\documentclass[12pt]{article}
\usepackage{amsmath,amsfonts,amsthm,mathrsfs,amsopn,amssymb,hyperref,xcolor,mathtools}
\usepackage{longtable} 
\usepackage{lineno} 
\usepackage{soul} 
\usepackage{authblk}
\newtheorem{myDef}{Definition}
\allowdisplaybreaks[4]

\usepackage{lipsum}
\usepackage{amsfonts}
\usepackage{graphicx}
\usepackage{epstopdf}
\usepackage{algorithmic}
\usepackage[title]{appendix}
\ifpdf
  \DeclareGraphicsExtensions{.eps,.pdf,.png,.jpg}
\else
  \DeclareGraphicsExtensions{.eps}
\fi

\makeatletter
\newcommand*{\addFileDependency}[1]{
  \typeout{(#1)}
  \@addtofilelist{#1}
  \IfFileExists{#1}{}{\typeout{No file #1.}}
}
\makeatother

\newtheorem*{mydef*}{Definition}
\newtheorem{theorem}{Theorem}

\newtheorem{lemma}{Lemma}
\newtheorem{prop}{Proposition}
\newtheorem*{prop*}{Proposition}

\newtheorem*{problem*}{Problem}

\newcommand{\CA}{\mathcal{A}}

\newcommand{\CL}{\mathcal{L}}

\newcommand{\CH}{\mathcal{H}}
\newcommand{\CG}{\mathcal{G}}
\newcommand{\CF}{\mathcal{F}}

\newcommand{\CW}{\mathcal{W}}

\newcommand{\BE}{\mathbb{E}}
\newcommand{\BR}{\mathbb{R}}

\newcommand{\SP}{\mathscr{P}}

\newcommand{\SD}{\mathscr{D}}
\newcommand{\sh}{_{\#}}
\newcommand{\argmin}{\mathop{\mathrm{argmin}}}

\newcommand{\NN}{\mathcal{NN}}

\newcommand{\dd}{\mathrm{d}}

\title{Theoretical Insights into CycleGAN: Analyzing Approximation and Estimation Errors in Unpaired Data Generation}

\author{SUN Luwei \thanks{SUN Luwei is with the Department of Mathematics, City University of Hong Kong, Kowloon, Hong Kong (email:luweisun2-c@my.cityu.edu.hk).}, ~SHEN Dongrui\thanks{SHEN Dongrui is with the Department of Mathematics, City University of Hong Kong, Kowloon, Hong Kong (email:dongrshen2-c@my.cityu.edu.hk).}~ and ~FENG Han\thanks{FENG Han is with the Department of Mathematics, City University of Hong Kong, Kowloon, Hong Kong (email:hanfeng@cityu.edu.hk).}}

\date{\vspace*{-1cm}}

\begin{document}

\maketitle

\begin{abstract}
    In this paper, we focus on analyzing the excess risk of the unpaired data generation model, called CycleGAN. Unlike classical GANs, CycleGAN not only transforms data between two unpaired distributions but also ensures the mappings are consistent, which is encouraged by the cycle-consistency term unique to CycleGAN.
    The increasing complexity of model structure and the addition of the cycle-consistency term in CycleGAN present new challenges for error analysis. By considering the impact of both the model architecture and training procedure, the risk is decomposed into two terms: approximation error and estimation error. These two error terms are analyzed separately and ultimately combined by considering the trade-off between them. Each component is rigorously analyzed; the approximation error through constructing approximations of the optimal transport maps, and the estimation error through establishing an upper bound using Rademacher complexity. Our analysis not only isolates these errors but also explores the trade-offs between them, which provides a theoretical insights of how CycleGAN's architecture and training procedures influence its performance.
\end{abstract}

\section{Introduction}
With the development of deep learning, Generative adversarial networks (\textbf{GANs})\cite{goodfellow2014generative,arjovsky2017wasserstein} have become popular for their remarkable contribution to the improvement of deep generative models and have received substantial interest in recent years. Compared to the classical density estimation methods, the GANs learn the data distribution by training a generator and a discriminator against each other. GANs-related models are popularly used in image synthesis, such as image generation\cite{jin2017towards,karras2017progressive,radford2015unsupervised} and translation\cite{Yoo2016PixelLevelDT,wang2016generative,wang2018high,zhang2017stackgan}. The image-to-image translation\cite{isola2017image} is learning the mapping from input one image to one output image by a training set with aligned input and output images. Recent studies in computer vision \cite{eigen2015predicting,johnson2016perceptual,zhang2016colorful} promote powerful improvement in image-to-image translation in the supervised setting. However, problems still exist with the limited paired training sets. In the practical scenario, the paired training data for the segmentation tasks is relatively small in areas such as medical images. Additionally, the paired training set is not defined for other novel translations, e.g., the translation in artistic style and object transfiguration. The limitations in paired training data lessen the flexibility in image-to-image translation. There is a further problem in solving the unpaired image-to-image translation, where no matches are provided between the training domains of the input and output. The Cycle-Consistent adversarial networks (\textbf{CycleGAN})\cite{zhu2017unpaired} provide a solution inspired by the structure of GANs. Traditional GANs train the generator to guarantee the target distribution with a given distribution (e.g., Gaussian distribution). Other than GANs, CycleGAN considers the translation between two unpaired datasets, which means constructing the generation between two unknown distributions. Lacking supervision in the paired training data sets, CycleGAN trains two inverse translators between two unpaired training sets and introduces the cycle consistency loss\cite{zhou2016learning} to confirm the two mappings are bijections. Without paired training examples, CycleGAN can identify unique characteristics of the input set of images and determine how these characteristics can be transformed to match the other set of images. Breaking the restrictions in training data, applications of CycleGAN are various, such as transferring the style or object of an image and image enhancement. This network is also applied to enhance the performance of the translation with insufficient paired datasets.

CycleGAN applies the property of cycle consistency to the translation model by combining two traditional GAN models to construct the structure of CycleGAN (see Figure \ref{pictureCG}).
\begin{figure}[htbp]
    \centering
    \includegraphics[width=.5\textwidth]{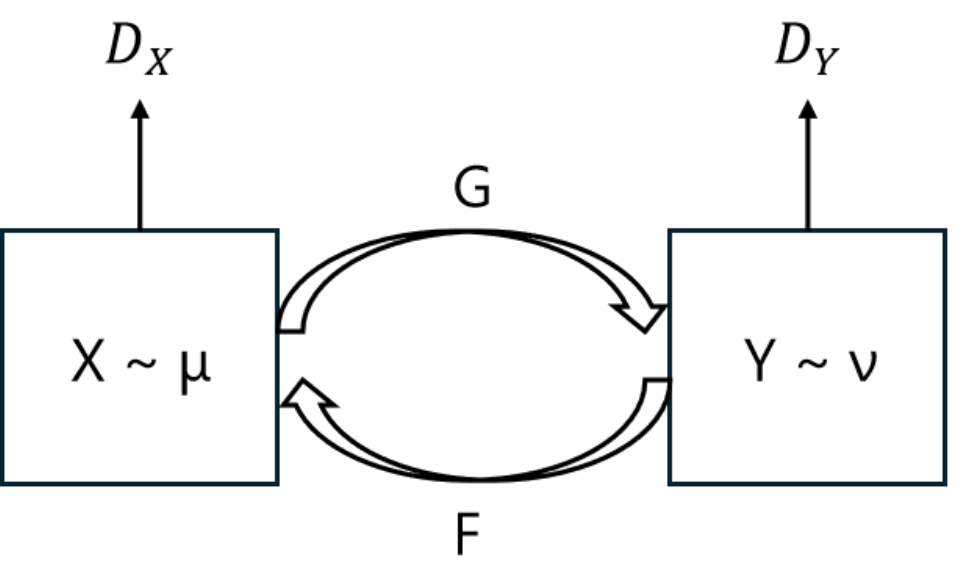}
    \caption{The general framework of CycleGAN\cite{zhu2017unpaired}.}
    \label{pictureCG}
\end{figure}
The model involves $X$ and $Y$ as the two data spaces and $\SP(X)$ and $\SP(Y)$ as the spaces of probability measures defined on $X$ and $Y$. The probability distribution at the end is denoted by $\mu \in \SP(X)$ and $\nu \in \SP(Y)$. CycleGAN defines the forward generation process mapping as $G: X \rightarrow Y$ and the backward generation process as $F: Y \rightarrow X$, accompanied by two discriminators $ D_{X} $ and $ D_{Y} $. CycleGAN considers the distance between the generated distribution and the target distribution measured by the corresponding discriminators as adversarial loss and involves cycle consistency loss to ensure the mappings are consistent. In this work, we define the distance under discriminator $D_X$ between target distribution $\mu$ and generated distribution $F_{\#}\nu$ with integral probability metrics (IPM) as,
$$d_{\SD_X}(\mu,F_\#\nu) = \sup\limits_{D_X\in\SD_X} \Bigl\{\BE_{\mathbf{x}\sim\mu}[D_X(\mathbf{x})]-\BE_{\mathbf{y}\sim\nu}[D_X(F(\mathbf{y}))] \Bigr\}.$$ We discuss the optimization task of CycleGAN training with the following formulation,
    \begin{equation}
        \label{eq:cycle-op}
		\inf _{F\in\CF, G\in\CG} L(F,G) = \inf _{F\in\CF, G\in\CG} \lambda \mathcal{L}_{\text {cyc }}(\mu, \nu, F, G)+d_{\SD_X}\left(\mu, F_{\#} \nu\right)+d_{\SD_Y}\left(\nu, G_{\#} \mu\right),
    \end{equation}
    where the pre-specified parameter $\lambda>0$ controls the relative importance of cycle consistency loss to adversarial loss. In practice, we use a set of training samples $\{x_i\}_{i=1}^{n}$ from $\mu$ and $\{y_i\}_{i=1}^{m}$ from $\nu$ to evaluate the empirical distribution $\hat\mu$ and $\hat\nu$. The training of CycleGAN solves the empirical risk as follows,
     \begin{equation}
        \label{eq:cycle-em}
		\inf _{F\in\CF, G\in\CG} \hat L(F,G) = \inf _{F\in\CF, G\in\CG} \lambda \mathcal{L}_{\text {cyc }}(\hat\mu, \hat\nu, F, G)+d_{\SD_X}\left(\hat\mu, F_{\#} \hat\nu\right)+d_{\SD_Y}\left(\hat\nu, G_{\#} \hat\mu\right).
    \end{equation}
    As in the practical scenarios, the training of the CycleGAN operates on empirical distributions, so it is important to learn the excess risk brought out from the training process. In our work, we denote the $\hat F, \hat G$ as the solution of CycleGAN training (Eq.\ref{eq:cycle-em}) and analyze the excess risk defined as,
    \begin{equation}
        \label{eq:cycle-excess}
		L(\hat F,\hat G)- \inf _{F\in\CF, G\in\CG} L(F,G).
    \end{equation}

The challenges associated with training GANs are well-documented. Researchers are continuously refining the architecture and training methods of GAN models. Innovative models such as StyleGAN\cite{karras2019style,karras2020analyzing,karras2020training} and R3GAN\cite{huang2025gandeadlonglive} have been developed, significantly improving training stability in practice. On the other side, recent studies have delved deeply into the theoretical understanding of GANs. Since CycleGAN is based on the GAN framework, it inspires further analysis of CycleGAN. Some studies analyze GANs and other generative models from the perspective of optimal transport\cite{lei-no-date,chakrabarty2022translation}. The approximation error of GANs can be defined by measuring the distance to the corresponding optimal transport map. By exploiting the regularity of optimal transport\cite{villani2003,ChenShibing2018Grft}, we can estimate the approximation error of GANs with the constructive approximation techniques using deep networks\cite{yarotsky2017error,ali2021approximation,lu2021deep,gribonval2022approximation,DeVoreRonald2021Nna}. Studies of estimation error focus on the generalization properties of GANs, which analyze the capacity of GANs to learn a distribution from finite samples. Researchers analyzed estimation errors in different ways. Some researchers consider estimation error to be the convergence rate of the well-trained GAN generator. Specifically, Zhang et al.\cite{zhang2017discrimination} consider the estimation error only with the impact of the discriminator. Liang\cite{liang2021well} shows the convergence rates of learning distributions with GANs, which are constructed with different discriminators and generators. Huang et al.\cite{huang2022error} focus on the convergence rate of the generator, which is the solution of GAN training. They control generator approximation and discriminator approximation errors by constructing neural networks for approximation, as well as statistical errors using the empirical process theory. The generator approximation error vanished with a sufficiently large generator network. Ji et al.\cite{ji2021understanding}consider estimation and generalization errors in GAN training via SGM. The upper bound of the estimation error is related to the training sample and the neural network complexity of both the generator and discriminator.

Our study analyses the similarity in construction between GANs and CycleGAN and investigates the risk estimation of CycleGAN. Traditional training in GANs involves an unknown target distribution and a simple known distribution. CycleGAN is trained on two unknown distributions with no matches between them. Two generators in CycleGAN obtain two inverse processes $G: X \rightarrow Y$ and $F: Y \rightarrow X$ and these two mappings are bijections. As we study the excess risk\cite{bach2021learning}, we consider the error between the solution of CycleGAN training $L(\hat F, \hat G)$ and unconstrained optimal risk $\inf _{F\in\CF, G\in\CG} L(F, G)$, which evaluates the efficiency of the models derived from the training data applying on unseen data. We decompose the excess risk into two parts: approximation error and estimation error. The approximation error characterizes the assumptions about the modeling approach taken by the selected class of functions. The estimation error is determined by the size of the training sample set and the characteristics of both the generator and discriminator networks. In the analysis of excess risk, approximation and estimation errors exhibit an interactive relationship.

In this paper, we give a theoretical explanation of the model to illuminate the accuracy of the CycleGAN. We reformulated and decomposed the excess risk of CycleGAN. We provide an upper bound of the excess risk by considering approximation and estimation errors. For the approximation error, we explore its connection to the approximation of optimal transport maps using deep ReLU networks. For the estimation error, we derive an upper bound using Rademacher complexity, which captures the interaction between the generators and discriminators during CycleGAN training. We further utilize the covering number to refine the bound on the estimation error and estimate the Rademacher complexity. We further discuss the trade-off between approximation and estimation errors to establish the upper bound for excess risk.

The structure of this paper is as follows. In Section \ref{sec:form}, we describe the setting of our CycleGAN model and the optimization task to solve. In Section \ref{sec:err}, we decompose the excess risk into approximation and estimation errors. Specifically:
\begin{itemize}
\item In Section \ref{sec:approx-error}, we show that the approximation error can be upper bounded by the error in approximating optimal transport maps by deep ReLU networks.
\item In Section \ref{sec:esti-error}, we present an upper bound on the estimation error using Rademacher complexity and the covering number.
\item In Section \ref{sec:excess-risk}, we combine the approximation and estimation error and given that the excess risk can be bounded by $O (N^{-\frac{\alpha}{3+2d}}(\log \frac{1}{\delta})^{\frac{1}{2}})$, where $N$ is related to the sizes of the training sets, with the structure of the generators and discriminators defined properly.
\end{itemize}

\section{Preliminaries}
\label{sec:form}

\paragraph{{ReLU Neural Networks}}Let $\NN(\CW,\CL, B)$ represent the collection of all neural networks $f:\BR^d\to\BR^{d'}$ with width $\CW$, depth $\CL$,  {and norm constraint $B$. We generalize the standard fully connected neural network class by allowing the activation function $\sigma$ to apply either the ReLU function or the identity mapping to each component of its input vector. This expanded model class contains traditional networks as a special case and also allows for skip connections, such as those found in ResNet. Formally, the expanded class is defined as:} 
\begin{equation}
    \begin{gathered}
        \NN(\CW,\CL,B) :=\{\CA_{\CL} \circ \sigma \circ \CA_{{\CL-1}} \circ \sigma \circ \cdots \circ \sigma \circ \CA_{1} \circ \sigma \circ \CA_{0}: \\ 
        \Vert{(\mathbf{A}_{\CL}},{\mathbf{b}_{\CL})}\Vert_{\infty}\prod_{\ell=0}^{\CL-1}\max\{\Vert{(\mathbf{A}_{\ell}},{\mathbf{b}_{\ell})}\Vert_{\infty},1\}\le B\}
    \end{gathered}    
\end{equation}
where  {$\CA_i(\mathbf{x}):=\mathbf{A}_i\mathbf{x}+\mathbf{b}_i$} for $i=0,\dots,\CL$ are affine transforms with trainable parameters, weight matrices  {$A_i\in\BR^{d_{i}\times d_{i-1}}$} and bias vector $\mathbf{b}_i\in\BR^{d_i}$ with $d_0=d,d_{\CL}=d'$, and the activation $\sigma$ will act on each element of input vectors. The width is given by  $\CW=\max\{d_i\}_{i=1}^{\CL-1}$. For simplicity, we assume $d_1=d_2=\cdots=d_{\CL-1}=\CW$ in this work.

\paragraph{CycleGAN}CycleGAN exploits the idea that by translating an image from one domain to another and then applying the reverse transformation, the original image should be recovered. The goal is to learn maps $G$ and $F$ that produce output images distributed as target domains $\nu\in\SP(Y)$ and $\mu\in\SP(X)$, respectively. We assume $X$ and $Y$ are compact in $\BR^d$ for $d\ge 1$, and both target distributions $\mu$ and $\nu$ are absolutely continuous. The translation maps $G$ and $F$ are trained with discriminator networks $D_Y$ and $D_X$ in the adversarial manner. Also, the cycle consistency loss is introduced to regularize the model. The loss function, denoted as $L(F,G)$, is a weighted sum of the translation adversarial loss and cycle-consistency loss:
\begin{equation}
\label{equ:L1}
    L(F,G) := \lambda\CL_{cyc}(\mu,\nu,F,G) + d_{\SD_X}(\mu,F_\#\nu) + d_{\SD_Y}(\nu,G_\#\mu)
\end{equation}
where the pre-specified parameter $\lambda>0$ controls the relative importance of cycle consistency loss to adversarial loss.

Let $F\sh\nu\in\SP(X)$ and $G\sh\mu\in\SP(Y)$ be the push-forward measures of the translation map $F$ and $G$, respectively. The adversarial losses of the backward and the forward translation processes are defined with the integral probability metrics (IPM) between the target measure and the push-forward measure:
\begin{equation}
\begin{aligned}
    d_{\SD_X}(\mu,F_\#\nu) &= \sup\limits_{D_X\in\SD_X} \Bigl\{\BE_{\mathbf{x}\sim\mu}[D_X(\mathbf{x})]-\BE_{\mathbf{y}\sim\nu}[D_X(F(\mathbf{y}))] \Bigr\} \\
    d_{\SD_Y}(\nu,G_\#\mu) &= \sup\limits_{D_Y\in\SD_Y} \Bigl\{\BE_{\mathbf{y}\sim\nu}[D_Y(\mathbf{y})]-\BE_{\mathbf{x}\sim\mu}[D_Y(G(\mathbf{x}))] \Bigr\}
\end{aligned}
\end{equation}
where $\SD_X$ and $\SD_Y$ denote the discriminator function classes that correspond to the target domains $X$ and $Y$, respectively.  Let $\SD_X$ and $\SD_Y$ be the class of 1-Lipschitz functions, then $d_{\SD_X}(\mu,F_\#\nu)$ and $d_{\SD_Y}(\nu,G_\#\mu)$ degenerate into the $W_1$-distance.

The cycle-consistency loss is defined as:
\begin{equation}
    \CL_{cyc}(\mu,\nu,F,G) := \BE_{\mathbf{x}\sim\mu}\Bigl[\Vert{\mathbf{x}-F(G(\mathbf{x}))}\Vert_{1}\Bigr] + \BE_{\mathbf{y}\sim\nu}\Bigl[\Vert{\mathbf{y}-G(F(\mathbf{y}))}\Vert_{1}\Bigr]
\end{equation}
where $\Vert\cdot\Vert_1$ denotes the $\ell_1$-norm.

Suppose we have $n$ i.i.d. samples $\{\mathbf{x}_i\}_{i=1}^{n}$ from $\mu$ and $m$ i.i.d. samples $\{\mathbf{y}_j\}_{j=1}^{m}$ from $\nu$. Then, we can define the empirical loss function:
\begin{equation}
\label{equ:L2}
     \hat{L}(F,G) :=\lambda\CL_{cyc}(\hat\mu,\hat\nu,F,G) + d_{\SD_X}(\hat\mu,F_\#\hat\nu) + d_{\SD_Y}(\hat\nu,G_\#\hat\mu)
\end{equation}
where $\hat\mu:=\frac{1}{n}\sum_i\delta_{\mathbf{x}_i}$ and $\hat\nu:=\frac{1}{m}\sum_j\delta_{\mathbf{y}_j}$ are the empirical distribution of $\mu$ and $\nu$. In learning theory \cite{bach2021learning}, $L(F, G)$ and $\hat{L}(F, G)$ are referred to as the expected risk and the empirical risk, respectively.

\paragraph{Assumptions on the structure of CycleGAN} In this paper, we consider the generator neural networks of maps $F$, $G$ as $\NN(\CW_F,\CL, B_F)$, $\NN(\CW_G,\CL, B_G)$ and the discriminator neural networks of maps $D_X$, $D_Y$ as $\NN(\CW_{D_X},\CL,1)$, $\NN(\CW_{D_Y},\CL, 1)$ respectively.

\section{Error Analysis}
\label{sec:err}
We consider the following expected risk and empirical risk minimization problems:
\begin{gather}
    \tilde{F},\tilde{G}:=\argmin\limits_{F,G~\text{ReLU}}L(F,G)\\
    \hat{F},\hat{G}:=\argmin\limits_{F,G~\text{ReLU}}\hat{L}(F,G)
\end{gather}
The excess risk of $\hat{F},\hat{G}$ is equal to $L(\hat{F},\hat{G}) - L^*$, where $ L^*=\inf\limits_{F,G}L(F,G)$ for all measurable $F$ and $G$. It measures how well the models learned from training data generalize to unseen data and could be decomposed into two terms as follows:
\begin{equation}
\label{eq:de}
L(\hat{F},\hat{G}) - L^*  =  \underbrace{L(\tilde{F},\tilde{G})-L^*}_{\textit{(1)approximation error}}  +  \underbrace{ [L(\hat{F},\hat{G})-L(\tilde{F},\tilde{G})]}_{\textit{(2)estimation error}}
\end{equation}

\subsection{Approximation Error}
\label{sec:approx-error}

In this section, we aim to establish an upper bound for the approximation error. Specifically, we exploit the connection between translation loss and cycle-consistency loss, which enables us to estimate $L(\tilde{F},\tilde{G})-L^*$ by constructing ReLU networks to approximate the optimal transport map that achieves $L^*$.

First, we recall the existence result of optimal transport problems as follows.
\begin{lemma}[Brenier's theorem, \cite{villani2003}]
\label{lemma:brenier}
Let $\mu,\nu$ be two probability measures on $\BR^d$, such that $\mu$ does not give mass to small sets (those ones with Hausdorff dimension are at most $d-1$). Then there is exactly one measurable map $T$ such that $T\sh\mu=\nu$ and $T=\nabla\varphi$ for some convex $\varphi$, in the sense that any two such maps coincide $\dd\mu$-almost everywhere.
\end{lemma}
For CycleGAN, we consider the cyclic transport problem between $\mu$ and $\nu$. Brenier's theorem guarantees the existence of the optimal transport $\mu\xrightarrow{~\nabla\varphi~}\nu\xrightarrow{~\nabla\psi~}\mu$ for some convex $\varphi$, $\psi$. Moreover, given the assumption that $\mu$ and $\nu$ have corresponding densities $f$ and $g$ with respect to Lebesgue measure, we can show that $\nabla\varphi,\nabla\psi$ are in H\"older classes $\CH^\alpha$, for $\alpha\in (1,2)$ \cite{ChenShibing2018Grft}. The existence of optimal transport map $\nabla\varphi$ and $\nabla\psi$ guarantees that the unconstrained optimal risk $L^*=\inf_{F,G} L(F,G)$ goes to 0. Thus, it only leaves to analyze $L(\tilde{F},\tilde{G})$.

\begin{lemma}[Approximation error decomposition]
\label{lemma:approx-decomp}
Assume there exists convex functions $\varphi,\psi$ such that $\nu=\nabla\varphi\sh\mu$ and $\mu=\nabla\psi\sh\nu$. Then, for any generator neural networks $F,G$,  we have:
    \begin{equation*}
        L(F,G) \le C \sum_{i=1}^{d}\Bigl[\Vert{\nabla\varphi_i-G_i}\Vert_{L_\infty(X)} + \Vert{\nabla\psi_i-F_i}\Vert_{L_\infty(Y)} \Bigr]
    \end{equation*}
     where $i$ denotes the i-th coordinate and $C$ is a constant independent of $F$ and $G$.
\end{lemma}

Lemma \ref{lemma:approx-decomp} shows that $L(\tilde{F},\tilde{G})$ can be further bounded by the approximation error of the forward and backward translation processes, enabling us to extend the approximation theorem for deep ReLU neural networks to the CycleGAN structure.
Recall that the optimal transport maps $\nabla\psi,\nabla\varphi$ are in H\"older classes $\CH^{\alpha}$. The $L_\infty$-approximation rate for H\"older functions has been obtained in \cite{jiao2023approximation} using wide neural networks with norm constraints. In particular, the optimal approximation rates using shallow neural networks has been discussed in \cite{yang2024shallow}. We thus develop an analogous result using deep neural networks, i.e. if $\alpha<(d+3)/2$ and $d>3$, we can construct ReLU neural networks  {$f\in\NN(d+2,\CL,B)$} such that:
\begin{equation*}
\sup _{h \in \mathcal{H}^\alpha} \|h-f\|_{L^{\infty}\left(\Omega\right)} \lesssim \CL^{-\frac{\alpha}{d}} \vee B^{-\frac{2 \alpha}{d+3-2 \alpha}}
\end{equation*}
where $X\lesssim Y$ (or $Y
\gtrsim X$) denotes the statement that $X\leq CY$ for some $C>0$. See details in Appendix \ref{sec:appx-appr}.

By combining the lemmas we have discussed so far, we derive the approximation error rate for CycleGAN.

\begin{theorem}
\label{thm:approx-error}
 Let $X, Y$ be the unit cube $[0,1]^d$ in $\mathbb{R}^d$ with $d>3$. We can construct ReLU networks $G$ and $F$ with norm constraint $B \ge 1$, width  {$\CW \ge d^2+2d$}, and depth $2\le\CL \leq B^{(2 d)/(d+3-2 \alpha)}$, such that:
\begin{equation}
    L(\tilde{F},\tilde{G})-L^* \leq O(\CL^{-\alpha/d})
\end{equation}
where $\alpha\in(1,2)$ depends upon the smoothness of the optimal transport maps.
\end{theorem}

\subsection{Estimation Error}
\label{sec:esti-error}
In this section, we provide an upper bound of the estimation error. As defined in Eq.\eqref{eq:de}, the estimation error ($L(\hat{F},\hat{G})-L(\tilde{F},\tilde{G})$) characterizes the difference between the empirically trained generators and the desired generators. To analyze this difference, we introduce the further decomposition of the estimation error.
\begin{prop}
\label{prop:dec}
    The estimation error defined as $L(\hat{F},\hat{G})-L(\tilde{F},\tilde{G})$ is controlled by two statistical errors:
    \begin{equation}
    \label{equ:L4}
       \begin{aligned}
        &L(\hat{F},\hat{G})-L(\tilde{F},\tilde{G}) \\
        &\leq  ~L(\hat{F},\hat{G}) - \hat{L}(\hat{F},\hat{G}) + \hat{L}(\tilde{F},\tilde{G}) - L(\tilde{F},\tilde{G})\\
        &=
        \Bigl[d_{\SD_X}(\mu,\hat{F}_\#\nu)-d_{\SD_X}(\hat\mu,\hat{F}_\#\hat\nu) + d_{\SD_X}(\hat\mu,\tilde{F}_\#\hat\nu)-d_{\SD_X}(\mu,\tilde{F}_\#\nu)\Bigr]\\
        &\quad +\Bigl[d_{\SD_Y}(\nu,\hat{G}_\#\mu)-d_{\SD_Y}(\hat\nu,\hat{G}_\#\hat\mu) + d_{\SD_Y}(\hat\nu,\tilde{G}_\#\hat\mu)-d_{\SD_Y}(\nu,\tilde{G}_\#\mu)\Bigr]\\
        &\quad + \lambda \Bigl[\CL_{cyc}(\mu,\nu,\hat{F},\hat{G}) - \CL_{cyc}(\hat\mu,\hat\nu,\hat{F},\hat{G}) +\CL_{cyc}(\hat\mu,\hat\nu,\tilde{F},\tilde{G}) - \CL_{cyc}(\mu,\nu,\tilde{F},\tilde{G}) \Bigr] .\\
\end{aligned}
    \end{equation}
\end{prop}

It leaves us to concentrate on two types of estimation error: cycle-consistency type and generalization type. For any prediction $(F,G)$,
\begin{itemize}
    \item Cycle-consistency Type: $\CL_{cyc}(\mu,\nu,F,G) - \CL_{cyc}(\hat\mu,\hat\nu,F,G)$
    \item Generalization Type: $d_{\SD_X}(\mu,\hat{F}_\#\nu)-d_{\SD_X}(\hat\mu,\hat{F}_\#\hat\nu)$ and $d_{\SD_Y}(\nu,\hat{G}_\#\mu)-d_{\SD_Y}(\hat\nu,\hat{G}_\#\hat\mu)$
\end{itemize}

According to the formulation of CycleGAN defined previously in  Section \ref{sec:form}, the generators and discriminators are described by the weight matrices and bias vectors. Since the parameters are constrained, we can provide the bounding of the generator and discriminator neural networks. We derive the upper bound of the estimation error via the Rademacher complexity. Utilizing the covering number of the generators' and discriminators' function classes to find a further estimation of the Rademacher complexity with Dudley’s entropy integral\cite{inbook,Dudley_2002}, we can describe the upper bound of the estimation error. We further find the upper bound of the covering number and get the bounding of estimation error with the training sample and the width and depth of the generators' and discriminators' networks. The proof can be found in Appendix \ref{section:1}.
\begin{theorem}
\label{thm:est2}
  Let $\mu, \nu$ be the target distribution over the compact domain $X$, $Y$ on $[0,1]^d$, given $n$ i.i.d training samples as $\left\{\mathbf{x}_i\right\}_{i=1}^{n}$ from $\mu$ and $m$ i.i.d training samples $\left\{\mathbf{y}_i\right\}_{i=1}^{m}$ from $\nu$. Let $\NN(\CW_{D_X},\CL, 1)$ and $\NN(\CW_{D_Y},\CL, 1)$ be the neural network of discriminators $D_X,D_Y$, $\NN(\CW_F,\CL,B_{F})$ and $\NN(\CW_G,\CL,B_{G})$ be the neural network of generators $F,G$ as defined in Section \ref{sec:form}. We define $\mathcal{W} := \max\{\mathcal{W}_{D_X},\mathcal{W}_{D_Y},\mathcal{W}_F,\mathcal{W}_G\}$ and $B:= \max\{ B_F, B_G\}$. Then, with a probability of $1-12 \delta$, $$L(\hat{F},\hat{G})-L(\tilde{F},\tilde{G})
          =O(B (\sqrt{\frac{\mathcal{W}^2\mathcal{L}}{m}}+\sqrt{\frac{\mathcal{W}^2\mathcal{L}}{n}}+\sqrt{\frac{\log \frac{1}{\delta}}{ m}}+ \sqrt{\frac{\log \frac{1}{\delta}}{ n}})).$$
\end{theorem}

\subsection{Upper Bound of Excess Risk}
\label{sec:excess-risk}
We have analyzed the bounding of the approximation and estimation errors and provided the results separately in Theorem \ref{thm:approx-error} and Theorem \ref{thm:est2}. Following the decomposition (Eq.\eqref{eq:de}), we can get the upper bound of the excess risk of CycleGAN.
\begin{theorem}
\label{thm3}
    Let $X, Y$ be the unit cube $[0,1]^d$ in $\mathbb{R}^d$ with $d>3$. Let $\mu, \nu$ denote the target distributions over $X$, $Y$, respectively. We consider $n$ i.i.d. training samples $\{\mathbf{x}_i\}_{i=1}^{n}$ drawn from $\mu$ and $m$ i.i.d. training samples $\{\mathbf{y}_i\}_{i=1}^{m}$ drawn from $\nu$. Let $\NN(\CW_{D_X},\CL, 1)$ and $\NN(\CW_{D_Y},\CL,1)$ be the neural network of discriminators $D_X,D_Y$, $\NN(\CW_F,\CL,B_{F})$ and $\NN(\CW_G,\CL,B_{G})$ be the neural network of generators $F,G$ as defined in Section \ref{sec:form}. We define $\mathcal{W} := \max\{\mathcal{W}_{D_X},\mathcal{W}_{D_Y},\mathcal{W}_F,\mathcal{W}_G\}$ and $B:= \max\{ B_F, B_G\}$. If push-forward mappings are in in H\"older classes $\CH^\alpha$, for $\alpha\in (1,2)$, then with a probability of $1-12 \delta$, for any  {$\CW\ge d^2+2d$}, and $N=\max\{m,n\}$, when $ B = N ^{\frac{d+3-2\alpha}{4d+6}}$ and  {$\mathcal{L} = N^{\frac{d}{2d+3}}$}, we have $$L(\hat{F},\hat{G}) - L^* \leq  O (N^{-\frac{\alpha}{3+2d}}(\log \frac{1}{\delta})^{\frac{1}{2}}).$$
\end{theorem}
\begin{proof}
    Following Eq.\eqref{eq:de}, with the result we get in Theorem \ref{thm:approx-error} and Theorem \ref{thm:est2}, we have
    \begin{equation*}
    \label{eq:excess1}
        L(\hat{G},\hat{F}) - L^*  = C_1(\CL^{-\alpha/d})+C_2(B\{\sqrt{\frac{\mathcal{W}^2\mathcal{L}}{m}}+\sqrt{\frac{\mathcal{W}^2\mathcal{L}}{n}}+\sqrt{\frac{\log \frac{1}{\delta}}{ m}}+ \sqrt{\frac{\log \frac{1}{\delta}}{ n}}\}).
    \end{equation*}
    We observe that as the depth $\mathcal{L}$ increases, the approximation error and estimation error behave in opposite directions. To establish a bound on the excess risk, we must find a balance between these two errors, which reveals the relationship between depth $\mathcal{L}$ and sample size $N$.
    
    We define $q$ such that when $B = \mathcal{L}^{\frac{d+3-2\alpha}{2d}}$ and $\mathcal{L}\geq N^q$,
    \begin{equation*}
        \begin{aligned}
            &B\sqrt{\frac{\mathcal{L}}{N}} \geq \CL^{-\alpha/d}
        \end{aligned}
    \end{equation*}

    Thus, we have $q = \frac{d}{2d + 3}$. Consequently, we can obtain the final result with a probability of $1 - 12\delta$ when $\mathcal{L} = N^{\frac{d}{2d + 3}}$,
    \begin{equation}
    \label{eq:excess2}
    \begin{aligned}
        L(\hat{G},\hat{F}) - L^*  &= C_1(\CL^{-\alpha/d})+C_2(\sqrt{\frac{\mathcal{W}^2\mathcal{L}}{N}}+\sqrt{\frac{\log \frac{1}{\delta}}{N}})\\
        &\leq O (N^{-\frac{\alpha}{3+2d}}(\log \frac{1}{\delta})^{\frac{1}{2}})
    \end{aligned}
    \end{equation}
\end{proof}
When analyzing the approximation and estimation errors independently, we observe that the depth $\mathcal{L}$ and norm constraint $B$ influences these two types of errors differently. We establish a balance and set $\mathcal{L}$ to $N^{\frac{d}{2d+3}}$ and $ B$ to $N ^{\frac{d+3-2\alpha}{4d+6}}$ . The convergence of the excess risk presented in Theorem \ref{thm3} suggests a framework for constructing efficient neural networks in CycleGAN, establishing a relationship between the network's depth and the sample size.

\section{Conclusion and Discussion}
\label{sec:result}
In the error analysis of CycleGAN, we take the unconstrained optimal risk into consideration instead of only focusing on the convergence of the estimation error. We analyze the excess risk of the CycleGAN, which characterizes the deviation of the solution we get from the training process $\hat F, \hat G$ with respect to the optimal risk for all measurable generators $F, G$. We decompose the excess risk and analyze it individually through approximation error and estimation error. By leveraging the regularity of the optimal transport of CycleGAN, we present a constructive approximation result in terms of neural network width and depth. For the analysis of estimation error, we mainly focus on the bounding of the statistical error and provide the bounding of the estimation error with the impact of the sample size and the neural network width and depth of the generators and the discriminators. The excess risk is influenced by both approximation and estimation errors. The results indicate that the depth of a neural network affects these errors in opposite ways. We establish a relationship between the size of the training samples and the neural network's depth to balance the two errors. Specifically, we demonstrate that when the width ($\CW$), depth ($\CL$) and norm constraint $B$ of the generators and discriminators in CycleGAN are defined such that $\CW \geq 2d^2 + 3d$, $B = N ^{\frac{d+3-2\alpha}{4d+6}}$ and $\CL = N^{\frac{d}{2d+3}}$, the excess risk of CycleGAN can be bounded by $ O (N^{-\frac{\alpha}{3+2d}}(\log \frac{1}{\delta})^{\frac{1}{2}})$ with probability $1 - 12\delta$. Consequently, we show that when the relationship between the depth $\CL$, norm constraint $B$ and the sample size $N$ is satisfied, the bound on the excess risk is primarily determined by the training sample size.

 {Our main result, Theorem \ref{thm3}, shows that by appropriately choosing network width, depth, and norm constraints, one can construct a CycleGAN adapted to a prescribed training sample size. This idea provides size-dependent architectural guidelines and quantifies the network’s ability to learn the underlying data distribution. Our convergence analysis follows standard statistical learning theory. Related analyses appear in \cite{liang2017well, huang2022error} restricted for Vanilla Single-domain GANs. In particular, Liang \cite{liang2017well} shows that, with appropriately chosen discriminator and generator architectures, GANs attain an upper bound of order ($N^{-\frac{\alpha+1}{2 \alpha+2+d}}$). Huang et al. (Theorem 5, \cite{huang2022error}) further show a convergence rate of order ($N^{-\beta / d} \vee N^{-1 / 2} \log ^{c(\beta, d)} N$) when the discriminator and generator depths and widths scale appropriately with the sample size (N) in evaluation class as Hölder class $\mathcal{H}^\beta\left(\mathbb{R}^d\right)$. These rates can be contrasted by fixing the sample size and varying the architectures. In our results, the excess risk of CycleGAN is bounded with probability at least ($1-12 \delta$) by $O\left(N^{-\alpha /(3+2 d)}(\log (1 / \delta))^{1 / 2}\right)$ when the network scale is appropriately designed by the sample size $N$. This finding aligns with the analytical framework and results established for GANs. However, CycleGAN’s convergence rate is not directly comparable to the rates established for standard GANs, because CycleGAN addresses unpaired image-to-image translation—a problem setting that differs from those typically analyzed for standard GANs. Because this work focuses on the standard CycleGAN \cite{zhu2017unpaired}, we defer potential architectural refinements and the incorporation of other well-studied GAN designs to future work aimed at improving convergence.}

\section*{Acknowledgement}
The authors thank the anonymous referees for their constructive comments and suggestions. We also thank Prof. Chenchen Mou for helpful discussions with him. This work is supported partially by the Research Grants Council of Hong Kong [Projects \#11306220 and \#11308121].

\newpage
\bibliographystyle{siamplain}
\bibliography{main}

\newpage
\begin{appendices}
\section{Proof}

\subsection{Proof for approximation error analysis}
\label{sec:appx-appr}
\paragraph{Remarks on Lemma \ref{lemma:brenier}: Optimal transport and Monge-Amp\`ere equation} The assumption that $\mu$ does not give mass to small sets is to guarantee the uniqueness of the optimal transport. Here we consider the case of the Lebesgue measure. Assume that $\dd\mu=f$ and $\dd\nu=g$. By \cite{villani2003}, we can show that $\varphi$ is a convex solution to a particular type of Monge-Ampe\`re equation. Since $\nu=\nabla\varphi\sh\mu$, for all bounded continuous test functions $\zeta$, we have:
    \begin{equation*}
        \int\zeta({\bf y})g({\bf y})\dd{\bf y}=\int\zeta(\nabla\varphi({\bf x}))f({\bf x})\dd{\bf x}
    \end{equation*}
Then we can perform the change of variables ${\bf y}=\nabla\varphi({\bf x})$ in the left hand side:
    \begin{equation*}
        \int\zeta({\bf y})g({\bf y})\dd{\bf y}=\int\zeta(\nabla\varphi({\bf x)})g(\nabla\varphi({\bf x}))|D^2\varphi({\bf x})|\dd{\bf x}
    \end{equation*}
Since $\zeta$ is arbitrary, it gives:
    \begin{equation}
        \label{eq:ma-brenier}
        f({\bf x})=g(\nabla\varphi({\bf x}))|D^2\varphi({\bf x})|
    \end{equation}
Eq.\eqref{eq:ma-brenier} is a specific example of the Monge-Amp\`ere equation. We rewrite it as follows:
    \begin{equation}
        \label{eq:ma-eq}
        |D^2\varphi({\bf x})|=F({\bf x})
    \end{equation}
subject to the boundary condition:
    \begin{equation}
        \label{eq:ma-bc}
        \nabla\varphi(\Omega)=\Omega^*
    \end{equation}
where $\Omega,\Omega^*$ are bounded convex domains in $\BR^d$ with $C^{1,1}$ boundary, and $F$ is a positive function. Note that the transport $\nu=\nabla\varphi\sh\mu$ is guaranteed by the natural boundary condition above. We are interested in the regularity of the Monge-Amp\`ere equation.

\begin{lemma}[The $C^{2,\tau}$ regularity for the Monge-Amp\`ere equation, \cite{ChenShibing2018Grft}]
    \label{lemma:C-regularity}
    Assume that $\Omega$ and $\Omega^*$ are bounded convex domains in $\mathbb{R}^n$ with $C^{1,1}$ boundary, and assume that $F \in C^\tau(\bar{\Omega})$ is positive for some $\tau \in(0,1)$. Let $u$ be a convex solution to Eq.\eqref{eq:ma-eq} and Eq.\eqref{eq:ma-bc}. Then we have the estimate:
    \begin{equation*}
        \|u\|_{C^{2, \tau}(\bar{\Omega})} \leq C
    \end{equation*}
    where $C$ is a constant depending only $d, \tau, f, \Omega$, and $\Omega^*$.
\end{lemma}

The smoothness of the optimal transport map will be utilized to obtain the neural network approximation error later. Now we prove the optimal risk $L^*$ is zero as a corollary.

The optimal loss is defined as:
    \begin{equation*}
        \begin{aligned}
            L^* =& \inf_{F,G} L(F,G) \\
            =& \lambda\,\CL_{cyc}(\mu,\nu,F,G) + d_{\SD_X}(\mu,F\sh\nu) + d_{\SD_Y}(\nu,G\sh\mu) \\
            =& \lambda\,\BE_{\mu}\Bigl[\Vert{{\bf x}-F(G({\bf x}))}\Vert_{1}\Bigr] + \lambda\,\BE_{\nu}\Bigl[\Vert{{\bf y}-G(F({\bf y}))}\Vert_{1}\Bigr]  \\
            &+d_{\SD_X}(\mu,F\sh\nu) + d_{\SD_Y}(\nu,G\sh\mu)
        \end{aligned}
    \end{equation*}

With Brenier's theorem, there exists convex functions $\psi,\varphi$ such that $\mu=\nabla\psi\sh\nu$ and $\nu=\nabla\varphi\sh\mu$ with the optimal transport cost. Let $G=\nabla\varphi$ and  $F=\nabla\psi$. It is straightforward that $d_{\CH}(\mu,F\sh\nu)=0$ and $d_{\CH}(\nu,G\sh\mu)=0$ for any function class $\CH$.  {The existence of the optimal map implies that $F \circ G$ must be the identity $\mu$-almost everywhere, as it pushes $\mu$ to itself. Similarly, $G \circ F$ is the identity $\nu$-almost everywhere. It follows directly that the cycle consistency loss term is zero.} 

\paragraph{Proof of Lemma \ref{lemma:approx-decomp}} Let $(X,\Vert\cdot\Vert_1,
\mu)$ and $(Y,\Vert\cdot\Vert_1,\nu)$ be two metric measure spaces in $\BR^d$ and let $\dd\mu=f$ and $\dd\nu=g$.

Recall that $D_X$ and $D_Y$ are 1-Lipschitz functions. The translation error is bounded as follows:
    \begin{equation*}
    \begin{aligned}
        d_{\SD_X}(\mu,F\sh\nu) &= \sup_{D_X\in\SD_X}\Bigl\{\BE_{\mu}[D_X({\bf x})]-\BE_{\nu}[D_X(F({\bf y}))]\Bigr\} \\
        &=  \sup_{D_X\in\SD_X} \BE_{\nu}\Bigl[D_X(\nabla\psi({\bf y}))-D_X(F({\bf y}))\Bigr] \\
        &\le \,\BE_{\nu} \Bigl[\Vert{\nabla\psi({\bf y})-F({\bf y})}\Vert_{1}\Bigr] \\
        &\le \sum_{i=1}^{d}\Vert{\nabla\psi_i-F_i}\Vert_{L_\infty(Y)}
    \end{aligned}
    \end{equation*}

We denote the Lipschitz constants of the optimal transport maps $\nabla\psi,\nabla\varphi$ by $B_{\nabla\psi}$ and $B_{\nabla\varphi}$, respectively. The Lipschitz continuous gradient condition is guaranteed by the regularity for the Monge-Amp\`ere equation. The cyclic error is bounded as follows:
    \begin{equation*}
    \begin{aligned}
        &\BE_{\mu}\Bigl[\Vert{{\bf x}-F(G({\bf x}))}\Vert_{1}\Bigr]\\
        &= \BE_{\mu}\Bigl[\Vert{{\bf x}-F(\nabla\varphi({\bf x}))+F(\nabla\varphi({\bf x}))-F(G({\bf x}))}\Vert_{1}\Bigr] \\
        &\le \BE_{\nu}\Bigl[\Vert{\nabla\psi({\bf y})-F({\bf y})}\Vert_{1}\Bigr] + \BE_{\mu}\Bigl[\Vert{F(\nabla\varphi({\bf x}))-F(G({\bf x}))}\Vert_{1}\Bigr] \\
        &\le \BE_{\nu}\Bigl[\Vert{\nabla\psi({\bf y})-F({\bf y})}\Vert_{1}\Bigr] + \BE_{\mu}\Big[\Vert F(\nabla\varphi({\bf x}))-\nabla\psi(\nabla\varphi({\bf x})) \\ 
        &\quad + \nabla\psi(\nabla\varphi({\bf x)}) - \nabla\psi(G({\bf x})) + \nabla\psi(G({\bf x}))  -F(G({\bf x})) \Vert_{1}\Big] \\
        &\le \BE_{\nu}\Bigl[\Vert{\nabla\psi({\bf y})-F({\bf y})}\Vert_{1}\Bigr] + \BE_{\mu}\Big[\Vert F(\nabla\varphi({\bf x}))-\nabla\psi(\nabla\varphi({\bf x}))\Vert_{1}\Big]\\
        &\quad + \BE_{\mu}\Big[\Vert{\nabla\psi(\nabla\varphi({\bf x})) - \nabla\psi(G({\bf x}))}\Vert_{1}\Big] + \BE_{\mu}\Big[\Vert{\nabla\psi(G({\bf x}))  -F(G({\bf x}))}\Vert_{1}\Big]\\
        &\le 3\sum_{i=1}^d\Vert\nabla\psi_i-F_i\Vert_{L_\infty(Y)} + B_{\nabla\psi}\sum_{i=1}^d \Vert\nabla\varphi_i-G_i\Vert_{L_\infty(X)}
    \end{aligned}
    \end{equation*}

Similarly, we can bound $d_{\SD_Y}(\nu,G\sh\mu)$ and $\BE_{\mu}\Bigl[\Vert{{\bf y}-G(F({\bf y}))}\Vert_{1}\Bigr]$. Adding up these upper bounds, we have:
    \begin{equation*}
        L(F,G) \le C \sum_{i=1}^d \Bigl[ \Vert{\nabla\psi_i-F_i}\Vert_{L_\infty(Y)}+\Vert{\nabla\varphi_i-G_i}\Vert_{L_\infty(X)}\Bigr]
    \end{equation*}
    where $C$ depends on the optimal transport maps $B_{\nabla\psi}$, $B_{\nabla\varphi}$.

\paragraph{Remarks on approximation by norm constrained deep neural networks}
We consider the deep ReLU neural networks mapping $\BR^d$ to $\BR$. Following the notations in \cite{siegel2023optimal} and \cite{jiao2023approximation}, we denote the affine map with weight matrix $A$ and bias $b$ by $\CA_{A,b}(x)=Ax+b$. Then the class of deep neural networks with width $\CW$ and depth $\CL$ is given by:
\begin{equation*}
    \NN(\CW,\CL) :=\{\CA_{\CL} \circ \sigma \circ \CA_{{\CL-1}} \circ \sigma \circ \cdots \circ \sigma \circ \CA_{1} \circ \sigma \circ \CA_{0} \}
\end{equation*}
where we denote $\CA_{A_\ell,b_\ell}$ as $\CA_\ell$ for simplicity, and where the weight matrices satisfy ${A}_L \in \BR^{1 \times \CW}, {A}_0 \in \BR^{\CW \times d}$, and ${A}_1, \ldots, {A}_{L-1} \in \BR^{\CW \times \CW}$, and the biases satisfy $b_0, \ldots, b_{L-1} \in \BR^{\CW}$ and $b_L \in \BR$.  {We allow the activation function $\sigma$ to apply either the ReLU function or the identity mapping to each component of its input vector.}

Next, we can define the norm constrained ReLU neural network $\NN(\CW,\CL,B)$ as a subclass of $\NN(\CW,\CL)$:
\begin{equation*}
    \begin{gathered}
        \NN(\CW,\CL,B) :=\{\CA_{\CL} \circ \sigma \circ \CA_{{\CL-1}} \circ \sigma \circ \cdots \circ \sigma \circ \CA_{1} \circ \sigma \circ \CA_{0}: \\ 
        \Vert{(A_{\CL}},{b_{\CL})}\Vert_{\infty}\prod_{\ell=0}^{\CL-1}\max\{\Vert{(A_{\ell}},{b_{\ell})}\Vert_{\infty},1\}\le B\}
    \end{gathered}    
\end{equation*}
where we denote $\CA_{W_\ell,b_\ell}$ as $\CA_\ell$ for simplicity as above. Particularly, we consider the function class of shallow neural networks discussed in \cite{yang2024shallow}:
\begin{equation*}
    \CF(N,M):=\left\{f(x)=\sum_{i=1}^N a_i \sigma\left(\left(x^{\top}, 1\right) v_i\right): \max_{i=1,\dots,n}\{\Vert{v_i}\Vert_{1}\}\sum_{i=1}^N\left|a_i\right| \leq M\right\}
\end{equation*}
where $v_i\in\BR^{d+1}$ and $a_i$ are real numbers.

We derive the inclusion property between the function classes of norm constrained ReLU neural networks in the following lemma.

\begin{lemma}
\label{lemma:inclusion}
For integer $N>0$ and real number $M>0$, assume that $N=\sum_{k=1}^{K}N_k$ and $n=\max\{N_1,\dots,N_K\}$, then we have:
 {
\begin{equation*}
    \mathcal{F}(N, M)\subset\NN(d+n+1,K,M)
\end{equation*}
}
When $n=1$, we have:
 {
\begin{equation*}
    \mathcal{F}(N, M)\subset\NN(d+2,N,M)
\end{equation*}
}
\end{lemma}

\begin{proof}
We will first prove the case of $n=1$. For any $f\in\CF(N,M)$, it can be written as $f(x)=\sum_{i=1}^N a_i \sigma\left(\left(x^{\top}, 1\right) v_i\right)$. For  {$k=1,\dots,N$, let $P_k=\Vert{v_k}\Vert_1$, $Q_k=\max\{P_1,\dots,P_k\}$, and $S_{k}=\sum_{i=1}^{k}|a_i|$. The norm of the shallow neural network is calculated as $Q_NS_N=\max\limits_{i=1,\dots,N}\{\Vert{v_i}\Vert_{1}\}\sum_{i=1}^N\left|a_i\right| \leq M$.}

We consider a deep neural network mapping $x$ to $f(x)$ with the following parameterization:
\begin{equation*}
\begin{array}{ll}
    f_1=\sigma((x^T,1)\frac{v_1}{P_1}), &  h_1=\frac{a_1}{S_1}f_1\\
    f_2=\sigma((x^T,1)\frac{v_2}{P_2}), &  h_2=\frac{Q_1}{Q_2}\frac{S_1}{S_2}h_1 + \frac{P_2}{Q_2}\frac{a_2}{S_2}f_2 \\
    &\vdots \\
    f_{N-1}=\sigma((x^T,1)\frac{v_{N-1}}{P_{N-1}}), &  h_{N-1}=\frac{Q_{N-2}}{Q_{N-1}}\frac{S_{N-2}}{S_{N-1}}h_{N-2} + \frac{P_{N-1}}{Q_{N-1}}\frac{a_{N-1}}{S_{N-1}}f_{N-1} \\
    f_N=\sigma((x^T,1)\frac{v_N}{P_N}), &  h_N=Q_{N-1}S_{N-1}h_{N-1} + P_Na_Nf_{N} {=f(x)}
\end{array}
\end{equation*}
This neural network has $N$ hidden layers. We divide these hidden neurons into three types: $d$ source channels to push forward $x$, one regular channel to compute $f_i$, and one collation channel to compute $h_i$ as the linear combination of $h_{i-1}$ and $f_i$.  {Thus, this neural network has the width $d+2$}. 

Observe that $|a_1/S_1|=1$,  {$\Vert{v_k}\Vert_1/P_k=1$},  {and for  {$k=2,\dots,N-1$},}
 {\begin{equation*}
    \left|\frac{Q_{k-1}S_{k-1}}{Q_kS_k}\right|+\left|\frac{P_ka_k}{Q_kS_k}\right| \le \frac{Q_k(S_{k-1}+a_k)}{Q_kS_k}=\frac{Q_kS_k}{Q_kS_k}=1.
\end{equation*}}
Thus, we have its norm bounded by  {$|Q_{N-1}S_{N-1}|+|P_Na_N|\le Q_NS_N\le M$}.
This means that  {$f\in\NN(d+2,N,M)$} which completes the proof of this simple case.  {The architecture of the neural network is shown in Figure \ref{pictureNN}.}

\begin{figure}[htbp]
    \centering
    \includegraphics[width=.5\textwidth]{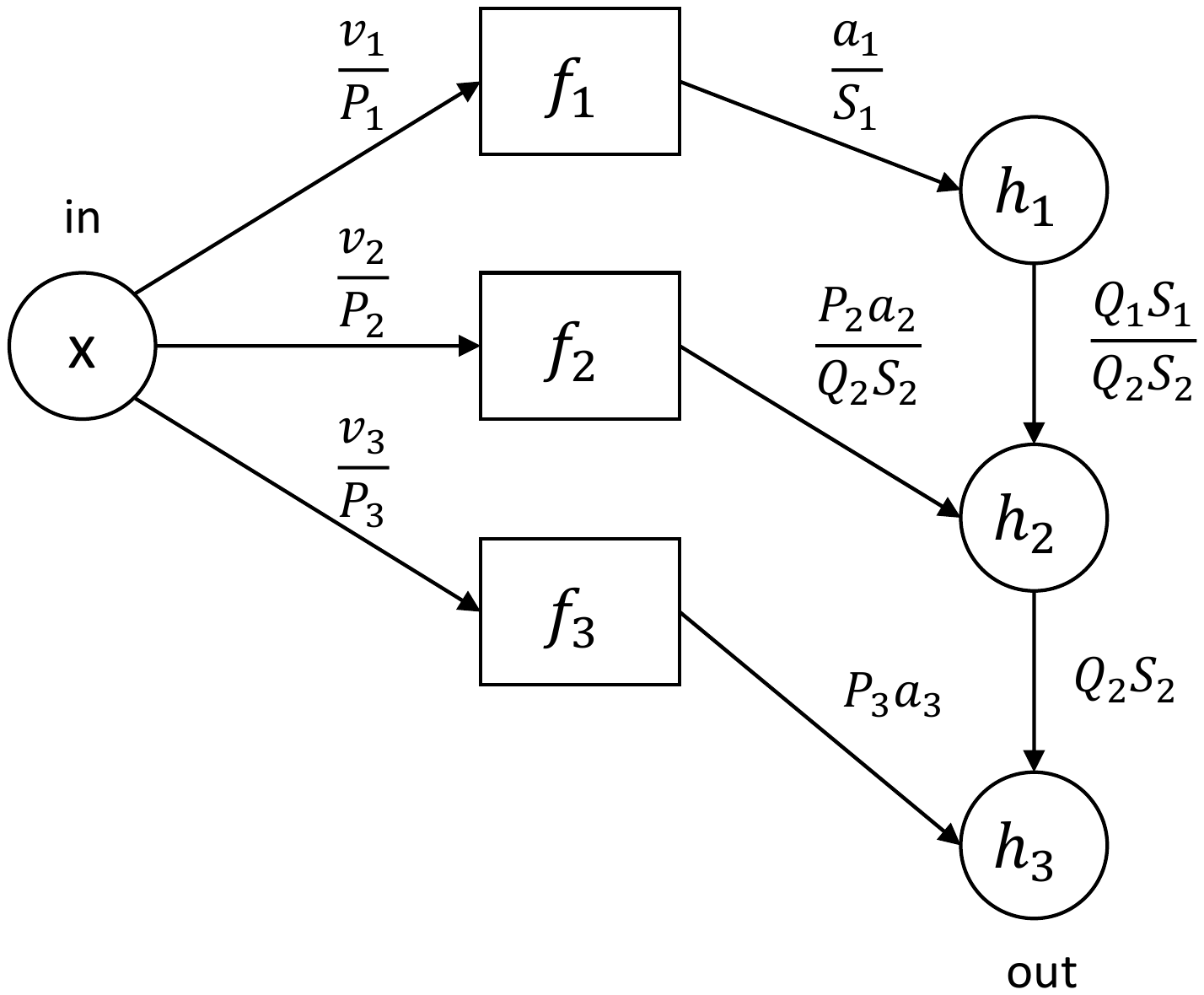}
    \caption{The framework of the defined neural network with 3 hidden layers.}
    \label{pictureNN}
\end{figure}

Next, we consider the case where $1\le n\le N$. Recall that $n=\max\{N_1,\dots,N_K\}$ and $N=\sum_{k=1}^KN_k$. For any $f\in\CF(N,M)$, we divide the sum of $N$ terms into $K$ groups and parameterize them as:
\begin{equation*}
    f(x)=\sum_{k=1}^K\sum_{i=1}^{N_k} a_i^{(k)} \sigma\left(\left(x^{\top}, 1\right) v_i^{(k)}\right)=\sum_{k=1}^K \sigma\left(\left(x^{\top}, 1\right) V_k^T\right)\alpha_k
\end{equation*}
where $V_k=(v_1^{(k)},\dots,v_{N_k}^{(k)})^T\in\BR^{N_k\times(d+1)}$ and $\alpha_k=(a_1^{(k)},\dots,a_{N_k}^{(k)})^T\in\BR^{N_k\times{1}}$.

For $k=1,\dots,K$, let $P_k=\Vert{V_k}\Vert_{\infty}$, $Q_k=\max\{P_1,\dots,P_k\}$, and $S_{k}=\sum_{i=1}^{k}\Vert{\alpha_i}\Vert_{1}$. By definition, we have:
\begin{equation*}
\begin{aligned}
    Q_KS_K &= \max_{k=1,\dots,K}\{\Vert{V_k}\Vert_{\infty}\}\sum_{k=1}^{K}\Vert{\alpha_k}\Vert_{1} \\
    &= \max_{k,i}\{\Vert{v_{i}}^{(k)}\Vert_1\}\sum_{k=1}^K\sum_{i=1}^{N_k}|a_{i}^{(k)}| \\
    &\le M
\end{aligned}
\end{equation*}

We construct the deep neural network mapping $x$ to $f(x)$ in a similar manner as above:
\begin{equation*}
\begin{array}{ll}
    f_1=\sigma((x^T,1)\frac{V_1^T}{P_1})^T, &  h_1=\frac{\alpha_1^T}{S_1}f_1\\
    f_2=\sigma((x^T,1)\frac{V_2^T}{P_2})^T, &  h_2=\frac{Q_1}{Q_2}\frac{S_1}{S_2}h_1 + \frac{P_2}{Q_2}\frac{\alpha_2^T}{S_2}f_2 \\
    &\vdots \\
    f_{K-1}=\sigma((x^T,1)\frac{V_{K-1}^T}{P_{K-1}})^T, &  h_{K-1}=\frac{Q_{K-2}}{Q_{K-1}}\frac{S_{K-2}}{S_{K-1}}h_{K-2} + \frac{P_{K-1}}{Q_{K-1}}\frac{\alpha_{K-1}^T}{S_{K-1}}f_{K-1} \\
    f_K=\sigma((x^T,1)\frac{V_K^T}{P_K})^T, &  h_K=Q_{K-1}S_{K-1}h_{K-1} + P_K\alpha^T_Kf_{K} {=f(x)}
\end{array}
\end{equation*}
Following a similar analysis as above, we can show this neural network is with depth $K$ and width  {$d+n+1$}, where we use $n$ regular channels to compute $f_k$. Its norm is bounded by $|Q_{K-1}S_{K-1}|+|P_K|\Vert\alpha_K\Vert_1\le Q_KS_K \le M$. This means that  {$f\in\NN(d+n+1,K,M)$}. Thus, we obtain the inclusion as desired.
\end{proof}

It was shown by Yang et al. \cite{yang2024shallow} that, if $\alpha<(d+3)/2$ and $d>3$, then:
\begin{equation*}
\sup _{h \in \mathcal{H}^\alpha} \inf _{f \in \CF(N,M)}\|h-f\|_{L^{\infty}\left(\Omega\right)} \lesssim N^{-\frac{\alpha}{d}} \vee M^{-\frac{2 \alpha}{d+3-2 \alpha}}
\end{equation*}
Next we apply the case of $n=1$ in Lemma \ref{lemma:inclusion} to obtain the approximation rate by the deep neural network class $\NN(\CW,\CL,B)$. Let  {$\CW\ge d+2$}, $\CL=N$ and $B=M$, if $\alpha<(d+3)/2$ and $d>3$, we can construct ReLU neural networks $f\in\CF(N,M)\subset\NN(\CW,\CL,B)$ such that:
\begin{equation*}
\sup _{h \in \mathcal{H}^\alpha} \|h-f\|_{L^{\infty}\left(\Omega\right)} \lesssim \CL^{-\frac{\alpha}{d}} \vee B^{-\frac{2 \alpha}{d+3-2 \alpha}}
\end{equation*}

\paragraph{Proof of Theorem \ref{thm:approx-error}}
Lemma \ref{lemma:C-regularity} guarantees that the optimal transport mappings $\nabla \psi_i, \nabla \varphi_i \in \CH^\alpha$, where $1<\alpha<2$, $i=1,\dots,d$. Then we have $\alpha<(d+3)/2$. As discussed above, for each $\nabla \psi_i, \nabla \varphi_i$, we can construct ReLU neural networks $F_i, G_i\in\NN(\CW,\CL,B)$ with  {$\CW \ge d+2$} such that:
\begin{equation*}
\begin{aligned}
& \left\|\nabla \psi_i-F_i\right\|_{L^{\infty}(Y)} \lesssim \mathcal{L}^{-\alpha / d} \vee B^{-\frac{2 \alpha}{d+3-2 \alpha}} \\
& \left\|\nabla \varphi_i-G_i\right\|_{L^{\infty}(X)} \lesssim \mathcal{L}^{-\alpha / d} \vee B^{-\frac{2 \alpha}{d+3-2 \alpha}}
\end{aligned}
\end{equation*}

We stack the networks $F_i$ and $G_i$ using parallelization to construct the neural networks $F$ and $G$, respectively. Then, the optimal transport map $\nabla \psi$ and $\nabla \varphi$ between $\nu$ and $\mu$ can be approximated by $\operatorname{ReLU}$ neural networks $F$ and $G$ with  {width $d^2+2d$} and depth $\mathcal{L}$. The rate $\mathcal{O}\left(\CL^{-\alpha / d}\right)$ holds when $B \gtrsim \CL^{(d+3-2 \alpha) /(2 d)}$.

\subsection{Proof for estimation error analysis}
\label{section:1}
\paragraph{Proof of Proposition \ref{prop:dec}} 
The estimation error could be further decomposed as follows:
\begin{equation*}
\label{equ:L3}
\begin{aligned}
        L(\hat{F},\hat{G})-L(\tilde{F},\tilde{G}) = & ~~~ \underbrace{L(\hat{F},\hat{G}) - \hat{L}(\hat{F},\hat{G})}_{\mathrm{I}} \\
        & \ + \underbrace{\hat{L}(\hat{F},\hat{G}) - \hat{L}(\tilde{F},\tilde{G})}_{\mathrm{II}} \\
        & \ + \underbrace{\hat{L}(\tilde{F},\tilde{G}) - L(\tilde{F},\tilde{G})}_{\mathrm{III}}
\end{aligned}
\end{equation*}
Since the empirical risk $\hat{L}(F,G)$ is minimized at $(\hat{F},\hat{G})$, we have part $(\mathrm{II})\le 0$.

We then focus on the approximation of $\bigl[L(\hat{F},\hat{G}) - \hat{L}(\hat{F},\hat{G})\bigr]$ and $\bigl[\hat{L}(\tilde{F},\tilde{G}) - L(\tilde{F},\tilde{G})\bigr]$ and further decompose them based on the definition in Eq.(\ref{equ:L1}) and Eq.(\ref{equ:L2}).
\begin{equation*}
\begin{aligned}
        & L(\hat{F},\hat{G})-L(\tilde{F},\tilde{G}) \\
        \leq & ~~~ \underbrace{L(\hat{F},\hat{G}) - \hat{L}(\hat{F},\hat{G})}_{\mathrm{I}} + \underbrace{\hat{L}(\tilde{F},\tilde{G}) - L(\tilde{F},\tilde{G})}_{\mathrm{III}}\\
        =&\lambda \Bigl[\CL_{cyc}(\mu,\nu,\hat{F},\hat{G}) - \CL_{cyc}(\hat\mu,\hat\nu,\hat{F},\hat{G}) +\CL_{cyc}(\hat\mu,\hat\nu,\tilde{F},\tilde{G}) - \CL_{cyc}(\mu,\nu,\tilde{F},\tilde{G}) \Bigr] \\
        &+ \Bigl[d_{\SD_X}(\mu,\hat{F}_\#\nu)-d_{\SD_X}(\hat\mu,\hat{F}_\#\hat\nu) + d_{\SD_X}(\hat\mu,\tilde{F}_\#\hat\nu)-d_{\SD_X}(\mu,\tilde{F}_\#\nu)\Bigr]\\
        &+\Bigl[d_{\SD_Y}(\nu,\hat{G}_\#\mu)-d_{\SD_Y}(\hat\nu,\hat{G}_\#\hat\mu) + d_{\SD_Y}(\hat\nu,\tilde{G}_\#\hat\mu)-d_{\SD_Y}(\nu,\tilde{G}_\#\mu)\Bigr]
\end{aligned}
\end{equation*}

For any prediction $(F,G)$, the statistical error $L(F,G) - \hat{L}(F,G)$ could be decomposed as follows.

\begin{equation*}
\begin{aligned}
    L(F,G) - \hat{L}(F,G) = & ~ \Bigl[\lambda\CL_{cyc}(\mu,\nu,F,G) + d_{\SD_Y}(\nu,G_\#\mu) + d_{\SD_X}(\mu,F_\#\nu)\Bigr] \\
    & - \Bigl[ \lambda\CL_{cyc}(\hat\mu,\hat\nu,F,G) - d_{\SD_Y}(\hat\nu,G_\#\hat\mu) - d_{\SD_X}(\hat\mu,F_\#\hat\nu) \Bigr] \\
    = & \,\lambda \Bigl[\CL_{cyc}(\mu,\nu,F,G) - \CL_{cyc}(\hat\mu,\hat\nu,F,G) \Bigr] \\
    & + \Bigl[  d_{\SD_Y}(\nu,G_\#\mu)-d_{\SD_Y}(\hat\nu,G_\#\hat\mu) \Bigr] \\
    & + \Bigl[ d_{\SD_X}(\mu,F_\#\nu)- d_{\SD_X}(\hat\mu,F_\#\hat\nu) \Bigr] \\
\end{aligned}
\end{equation*}
\paragraph{Proof of Theorem \ref{thm:est2}}  In general, we define the statistical error $\mathbb{E}[d_\mathcal{H}(\mu,\hat\mu)]$ which describes the distance of empirical distribution $\hat\mu$ and the true data distribution $\mu$ with function class $\mathcal{H}$. The decomposition of the estimation error shows that we should focus on the statistical error to analyze the estimation error. The bounding of the statistical error follows the standard strategy. We first describe the upper bound of the statistical error by Rademacher complexity. Then, we bound the Rademacher complexity utilizing the covering number of $\mathcal{H}$. Two main tools, Rademacher complexity and covering number, are involved in our study of estimation error, and we here give the definitions of them.
\begin{myDef}[Rademacher Complexity \cite{bartlett-no-date}]
    Let $\SD:=\{l(\mathbf{x})\}$ be a function class. Then, the Rademacher complexity $\mathcal{R}\left(\SD\right)$ is defined as
     $$\mathcal{R}\left(\SD\right)=\mathbb{E}_{\mathbf{x}, \epsilon} \sup _{l \in \SD}\left|\frac{1}{n} \sum_{i=1}^{n} \epsilon_i l\left( \mathbf{x}_i\right)\right|,$$

    where $\epsilon_1, \epsilon_2 \ldots, \epsilon_n$ are independent random variables uniformly chosen from $\{-1,1\}$. Similarly, for compositional function class $\mathcal{H}_{\SD \times \CF}:=\{l( f(\mathbf{x})):l \in \SD, f \in \CF \}$, the Rademacher complexity $\mathcal{R}\left(\mathcal{H}_{\SD \times \CF}\right)$ is defined as,

    $$\mathcal{R}\left(\mathcal{H}_{\SD \times \CF}\right)=\mathbb{E}_{\mathbf{x}, \epsilon} \sup _{l \in \SD, f \in \CF }\left|\frac{1}{n} \sum_{i=1}^{n} \epsilon_i l\left(f\left(\mathbf{x}_i\right)\right)\right| $$.

\end{myDef}
\begin{myDef}[Covering number]
    Let $(S, \rho)$ be a metric space, and let $T \subset S$. We say that $T^{\prime} \subset S$ is an $\alpha$-cover for $T$ if, for all $x \in T$, there exists $y \in T^{\prime}$ such that $\rho(x, y) \leq \alpha$. The $\alpha$-covering number of $(T, \rho)$, denoted $\mathcal{N}(\alpha, T, \rho)$ is the size of the smallest $\alpha$-covering.
\end{myDef}
We express the upper bound of the statistical error $\mathbb{E}[d_\mathcal{H}(\mu,\hat\mu)]$ with Lemma \ref{lemma:statistical error} \cite{huang2022error}, which follows the strategy as bounding the statistical error of function $\mathcal{H}$ by Rademacher complexity and bound the Rademacher complexity via covering number by Dudley’s entropy integral \cite{inbook,Dudley_2002}.
\begin{lemma}[Statistical error bounding \cite{huang2022error}]
\label{lemma:statistical error}

    Suppose $\sup _{h \in \mathcal{H}}\|h\|_{\infty} \leq B$, then we can bound $\mathbb{E}[d_\mathcal{H} (\mu,\hat\mu)]$ as,
    \begin{equation*}
    \begin{aligned}
\mathbb{E}_{\hat{\mathbf{x}}}\left[d_{\mathcal{H}}\left(\mu, \widehat{\mu}\right)\right] \leq 2 \mathbb{E}_{\hat{\mathbf{x}}} \inf _{0<\delta<B / 2}\left(4 \delta+\frac{12}{\sqrt{n}} \int_\delta^{B / 2} \sqrt{\log \mathcal{N}\left(\varepsilon, \mathcal{H}_{\left.\right|_{\hat{\mathbf{x}}}}, \|\cdot\|_{\infty}\right)} d \epsilon\right),
\end{aligned}
\end{equation*}
where we denote $\mathcal{H}_{\left.\right|_{\hat{\mathbf{x}}}}=\left\{\left(h\left(\mathbf{x}_1\right), \ldots, h\left(\mathbf{x}_n\right)\right): h \in \mathcal{H}\right\}$ for any i.i.d. samples $ \mathbf{\hat x}=\left\{\mathbf{x}_i\right\}_{i=1}^n$ from $\mu$ and $\mathcal{N}\left(\epsilon, \mathcal{H}_{\left.\right|_{\hat{\mathbf{x}}}},\|\cdot\|_{\infty}\right)$ is the $\epsilon$-covering number of $\mathcal{H}_{\left.\right|_{\hat{\mathbf{x}}}} \subseteq \mathbb{R}^d$ with respect to the $\|\cdot\|_{\infty}$ distance.
\end{lemma}

Next, we show the upper bound of the estimation error for CycleGAN. Following the decomposition of the estimation error (Prop.\ref{prop:dec}), we analyze the upper bound of the estimation error of CycleGAN in the generalization and cycle-consistency types, respectively. For the generalization error, we develop the upper bound of the backward generation process with a similar strategy in Lemma \ref{lemma:statistical error} and can also achieve the error of the forward process referring to the symmetric design of the CycleGAN structure.
\begin{lemma}[Estimation Error in Generalization Type]
\label{lemma:Generalization}
    Let $\mu, \nu$ be the target distribution over the compact domain $X$, $Y$ on $[0,1]^d$, given $n$ i.i.d training samples as $\left\{\mathbf{x}_i\right\}_{i=1}^{n}$ from $\mu$ and $m$ i.i.d training samples $\left\{\mathbf{y}_i\right\}_{i=1}^{m}$ from $\nu$. Let $\SD_X = \mathcal{N}\mathcal{N}(\mathcal{W}_{D_X},\mathcal{L},1)$ be the function class of discriminator $D_X$ and ${\CF=\mathcal{N}\mathcal{N}(\mathcal{W}_F,\mathcal{L}},B_{F})$ be the function class of generator $F$ as defined in Section \ref{sec:form}. We denote that $\sup _{l_x \in \SD_X}\|l_x\|_{\infty} \leq 1$ and $\sup _{l_x \in \SD_X, f \in \CF}\|l_x\circ f\|_{\infty} \leq B_F$. Then, we can get the upper bound with a probability of $1-4 \delta$ (where $\delta=\min\{\delta_1, \delta_2\}$),
    \begin{equation*}
    \begin{aligned}
        \CL_{\SD_X, \CF}(\mu,\nu) &\leq  16 \mathbb{E}_{\mathbf{\hat y}} \inf _{0<{\xi_{1}}< {B_F} / 2}\left( {\xi_{1}}+\frac{3}{\sqrt{m}} \int_{\xi_1}^{ {B_F} / 2} \sqrt{\log \mathcal{N}\left(\epsilon, {{\SD_{X}} \circ \CF}_{\left.\right|_{\mathbf{\hat y}}}, \|\cdot\|_{\infty}\right)} d \epsilon\right) \\
        &\quad + 32 \mathbb{E}_{\mathbf{\hat x}} \inf _{0<{\xi_{2}}< 1 / 2}\left({\xi_{2}}+\frac{3}{\sqrt{n}} \int_{\xi_2}^{1 / 2} \sqrt{\log \mathcal{N}\left(\epsilon, {\SD_X}_{\left.\right|_{\mathbf{\hat x}}}, \|\cdot\|_{\infty}\right)} d \epsilon\right)\\
  &\quad +2 B_F \sqrt{\frac{2\log \frac{1}{\delta_1}}{ m}}+ 2 \sqrt{\frac{2\log \frac{1}{\delta_2}}{ n}})
    \end{aligned}
    \end{equation*}

    $$    \begin{aligned}
		\text{where} \quad &\CL_{\SD_X,\CF}(\mu,\nu) \coloneqq d_{\SD_X}(\mu,\hat{F}_\#\nu)-d_{\SD_X}(\hat\mu,\hat{F}_\#\hat\nu) + d_{\SD_X}(\hat\mu,\tilde{F}_\#\hat\nu)-d_{\SD_X}(\mu,\tilde{F}_\#\nu).
	\end{aligned}$$

\end{lemma}
\begin{proof}
\label{proof:Generalization}
    \begin{equation*}
        \CL_{\SD_X,\CF}(\mu,\nu) \coloneqq d_{\SD_X}(\mu,\hat{F}_\#\nu)-d_{\SD_X}(\hat\mu,\hat{F}_\#\hat\nu) + d_{\SD_X}(\hat\mu,\tilde{F}_\#\hat\nu)-d_{\SD_X}(\mu,\tilde{F}_\#\nu).
    \end{equation*}
    For $d_{\SD_X}(\mu,\hat{F}_\#\nu)-d_{\SD_X}(\hat\mu,\hat{F}_\#\hat\nu)$ , we can write it as,
    \begin{equation*}
        \begin{aligned}
		d_{\SD_X}(\mu,\hat{F}_\#\nu)-d_{\SD_X}(\hat\mu,\hat{F}_\#\hat\nu)= ~&d_{\SD_X}(\mu,\hat{F}_\#\nu)-d_{\SD_X}(\mu,\hat{F}_\#\hat\nu)+d_{\SD_X}(\mu,\hat{F}_\#\hat\nu)-d_{\SD_X}(\hat\mu,\hat{F}_\#\hat\nu)\\
		\leq~& \sup _{\text{$l_x \in \SD_X$}}{\lvert \frac{1}{m} \sum_{i=1}^{m} l_x\left(\hat{f}(\mathbf{y}_i)\right)-\mathbb{E}_{\hat{f}_{\#} \nu} [l_x (\hat{f}(\mathbf{y}))]\rvert}\\
		&+\sup _{\text{$l_x \in \SD_X$}}\lvert \mathbb{E}_{\mu}[ l_x(\mathbf{x})]-\frac{1}{n} \sum_{i=1}^{n} l_x\left(\mathbf{x}_i\right)\rvert\\
        \leq~& \sup _{\text{$l_x \in \SD_X$,$f \in \CF$}}{\lvert \frac{1}{m} \sum_{i=1}^{m} l_x\left(f(\mathbf{y}_i)\right)-\mathbb{E}_{f_{\#} \nu} [l_x (f(\mathbf{y}))]\rvert}\\
		&+\sup _{\text{$l_x \in \SD_X$}}\lvert \mathbb{E}_{\mu}[ l_x(\mathbf{x})]-\frac{1}{n} \sum_{i=1}^{n} l_x\left(\mathbf{x}_i\right)\rvert
	\end{aligned}
    \end{equation*}
    For $d_{\SD_X}(\hat\mu,\tilde{F}_\#\hat\nu)-d_{\SD_X}(\mu,\tilde{F}_\#\nu)$, we can write it as,
    \begin{equation*}
        \begin{aligned}
		d_{\SD_X}(\hat\mu,\tilde{F}_\#\hat\nu)-d_{\SD_X}(\mu,\tilde{F}_\#\nu)= &d_{\SD_X}(\hat\mu,\tilde{F}_\#\hat\nu)-d_{\SD_X}(\mu,\tilde{F}_\#\hat\nu)+d_{\SD_X}(\mu,\tilde{F}_\#\hat\nu)-d_{\SD_X}(\mu,\tilde{F}_\#\nu)\\
		\leq& \sup _{\text{$l_x \in \SD_X$}}{\lvert\frac{1}{m} \sum_{i=1}^{m} l_x\left(\tilde{f}(\mathbf{y}_i)\right)-\mathbb{E}_{\tilde{f}_{\#} \nu} [l_x (\tilde{f}(\mathbf{y}))]\rvert}\\
		&+\sup _{\text{$l_x \in \SD_X$}}\lvert\mathbb{E}_{\mu}[ l_x(\mathbf{x})]-\frac{1}{n} \sum_{i=1}^{n} l_x\left(\mathbf{x}_i\right)\rvert
	\end{aligned}
    \end{equation*}
    Thus, we can get
    \begin{equation*}
        \begin{aligned}
		\CL_{\SD_X,\CF}(\mu,\nu)\leq &\sup _{\text{$l_x \in \SD_X$,$f \in \CF$}}{\lvert \frac{1}{m} \sum_{i=1}^{m} l_x\left(f(\mathbf{y}_i)\right)-\mathbb{E}_{f_{\#} \nu} [l_x (f(\mathbf{y}))]\rvert}\\
		&+\sup _{\text{$l_x \in \SD_X$}}{\lvert\frac{1}{m} \sum_{i=1}^{m} l_x\left(\tilde{f}(\mathbf{y}_i)\right)-\mathbb{E}_{\tilde{f}_{\#} \nu} [l_x (\tilde{f}(\mathbf{y}))]\rvert}\\
		&+2\sup _{\text{$l_x \in \SD_X$}}\lvert\mathbb{E}_{\mu}[ l_x(\mathbf{x})]-\frac{1}{n} \sum_{i=1}^{n} l_x\left(\mathbf{x}_i\right)\rvert.
	\end{aligned}
    \end{equation*}
    As $\SD_X $ and ${\CF}$ are bounded, we can apply McDiarmid's inequality \cite{doob-1940} to further bound $\CL_{\SD_X,\CF}(\mu,\nu)$, and get that with a probability of $1-4 \delta$ (where $\delta=\min\{\delta_1, \delta_2\}$)
    \begin{equation}
    \label{eqn:sample1}
        \begin{aligned}
		\CL_{\SD_X,\CF}(\mu,\nu)
           \leq& 2 B_F \sqrt{\frac{2\log \frac{1}{\delta_1}}{ m}}+ 2  \sqrt{\frac{2\log \frac{1}{\delta_2}}{ n}}\\
           &+ 2\underbrace{\mathbb{E}_{\mathbf{\hat y}}\sup _{\text{$l_x \in \SD_X$,$f \in \CF$}}{\lvert \frac{1}{m} \sum_{i=1}^{m} l_x\left(f(\mathbf{y}_i)\right)-\mathbb{E}_{f_{\#} \nu} [l_x (f(\mathbf{y}))]\rvert}}_{(\mathrm{I})}\\
           &+4\underbrace{\mathbb{E}_{\mathbf{\hat x}}\sup _{\text{$l_x \in \SD_X$}}\lvert\mathbb{E}_{\mu}[ l_x(\mathbf{x})]-\frac{1}{n} \sum_{i=1}^{n} l_x\left(\mathbf{x}_i\right)\rvert}_{(\mathrm{II})} .
	\end{aligned}
    \end{equation}
    We next estimate the upper bound of part $(\mathrm{I})$ and $(\mathrm{II})$ of Eq.(\ref{eqn:sample1}). As defined in (\ref{sec:form}), we have $\sup _{l_x \in \SD_X}\|l_x\|_{\infty} \leq 1$ and $\sup _{l_x \in \SD_X, f \in \CF}\|l_x\circ f\|_{\infty} \leq B_F$ Considering the result from Lemma \ref{lemma:statistical error}, we can easily get that for $(\mathrm{I})$ and $(\mathrm{II})$,
    \begin{equation*}
    \begin{aligned}
(\mathrm{I}) :\mathbb{E}_{\mathbf{\hat y}}&\sup _{\text{$l_x \in \SD_X$,$f \in \CF$}}{\lvert \frac{1}{m} \sum_{i=1}^{m} l_x\left(f(\mathbf{y}_i)\right)-\mathbb{E}_{f_{\#} \nu} [l_x (f(\mathbf{y}))]\rvert}\\
&\leq  2 \mathbb{E}_{\mathbf{\hat y}} \inf _{0<{\xi_{1}}< {B_F} / 2}\left(4 {\xi_{1}}+\frac{12}{\sqrt{m}} \int_{\xi_1}^{ B_F/ 2} \sqrt{\log \mathcal{N}\left(\epsilon, {{\SD_{X}} \circ \CF}_{\left.\right|_{\mathbf{\hat y}}}, \|\cdot\|_{\infty}\right)} d \epsilon\right),
\end{aligned}
\end{equation*}
\begin{equation*}
\begin{aligned}
(\mathrm{II}) :\mathbb{E}_{\mathbf{\hat x}}\sup _{\text{$l_x \in \SD_X$}}&\lvert\mathbb{E}_{\mu}[ l_x(\mathbf{x})]-\frac{1}{n} \sum_{i=1}^{n} l_x\left(\mathbf{x}_i\right)\rvert\\
\leq & 2 \mathbb{E}_{\mathbf{\hat x}} \inf _{0<{\xi_{2}}< 1 / 2}\left(4 {\xi_{2}}+\frac{12}{\sqrt{n}} \int_{\xi_2}^{1 / 2} \sqrt{\log \mathcal{N}\left(\epsilon, {\SD_X}_{\left.\right|_{\mathbf{\hat x}}}, \|\cdot\|_{\infty}\right)} d \epsilon\right).
\end{aligned}
\end{equation*}

\end{proof}

We next provide the bounding of cycle-consistency error following a similar strategy.
\begin{lemma}[Estimation Error in Cycle-consistency Type]
\label{lemma:Cycle}
   Let $\mu, \nu$ be the target distribution over the compact domain $X$, $Y$ on $[0,1]^d$, given $n$ i.i.d training samples as $\left\{\mathbf{x}_i\right\}_{i=1}^{n}$ from $\mu$ and $m$ i.i.d training samples $\left\{\mathbf{y}_i\right\}_{i=1}^{m}$ from $\nu$. Let ${\CF=\mathcal{N}\mathcal{N}(\mathcal{W}_F,\mathcal{L}},B_{F})$ and ${\CG=\mathcal{N}\mathcal{N}(\mathcal{W}_G,\mathcal{L}},B_{G})$ be the function classes of generators $F$, $G$ as defined in Section \ref{sec:form}. We denote that $\sup _{g \in \CG, f \in \CF}\|f\circ g\|_{\infty} \leq B_FB_G$ and $\sup _{g \in \CG, f \in \CF}\|g\circ f\|_{\infty} \leq B_GB_F$. Then, we can get the upper bound with a probability of $1-4 \delta$ (where $\delta=\min\{\delta_1, \delta_2\}$),
    \begin{equation*}
    \begin{aligned}
         \CL_{\CF,\CG}(\mu,\nu)&\leq 16 \mathbb{E}_{\mathbf{\hat y}} \inf _{0<{\xi_{1}}< {B_GB_F} / 2}\left( {\xi_{1}}+\frac{3}{\sqrt{m}} \int_{\xi_1}^{ {B_GB_F} / 2} \sqrt{\log \mathcal{N}\left(\epsilon, {\CG \circ \CF}_{\left.\right|_{\mathbf{\hat y}}}, \|\cdot\|_{\infty}\right)} d \epsilon\right) \\
         &\quad +16 \mathbb{E}_{\mathbf{\hat x}} \inf _{0<{\xi_{2}}< {B_FB_G} / 2}\left( {\xi_{2}}+\frac{3}{\sqrt{n}} \int_{\xi_2}^{ {B_FB_G} / 2} \sqrt{\log \mathcal{N}\left(\epsilon, {\CF \circ \CG}_{\left.\right|_{\mathbf{\hat x}}}, \|\cdot\|_{\infty}\right)} d \epsilon\right)\\
         &\quad +2 B_FB_G \sqrt{\frac{2 \log \frac{1}{\delta_1}}{n}}+2 {B_GB_F} \sqrt{\frac{2 \log \frac{1}{\delta_2}}{m}},
    \end{aligned}
    \end{equation*}
    where $\CL_{\CF,\CG}(\mu,\nu)\coloneqq \CL_{cyc}(\mu,\nu,\hat{F},\hat{G}) - \CL_{cyc}(\hat\mu,\hat\nu,\hat{F},\hat{G}) +\CL_{cyc}(\hat\mu,\hat\nu,\tilde{F},\tilde{G}) - \CL_{cyc}(\mu,\nu,\tilde{F},\tilde{G})$.
 \end{lemma}
   \begin{proof}
   \label{proof:Cycle}
   \begin{equation*}
           \CL_{\CF,\CG}(\mu,\nu)\coloneqq \CL_{cyc}(\mu,\nu,\hat{F},\hat{G}) - \CL_{cyc}(\hat\mu,\hat\nu,\hat{F},\hat{G}) +\CL_{cyc}(\hat\mu,\hat\nu,\tilde{F},\tilde{G}) - \CL_{cyc}(\mu,\nu,\tilde{F},\tilde{G})
       \end{equation*}
       For $\CL_{cyc}(\mu,\nu,\hat{F},\hat{G}) - \CL_{cyc}(\hat\mu,\hat\nu,\hat{F},\hat{G})$, we can write it as,
       \begin{equation*}
    \begin{aligned}
        \CL_{cyc}(\mu,\nu,\hat{F},\hat{G}) - \CL_{cyc}(\hat\mu,\hat\nu,\hat{F},\hat{G}) 
        = &\Bigl[\BE_{\mathbf{x}\sim\mu}[\Vert{\mathbf{x}-\hat{F}(\hat{G}(\mathbf{x}))}\Vert] + \BE_{\mathbf{y}\sim\nu}[\Vert{\mathbf{y}-\hat{G}(\hat{F}(\mathbf{y}))}\Vert]\Bigr]\\
        & -\Bigl[\frac{1}{n}\sum_i\Vert{\mathbf{x}_i-\hat{F}(\hat{G}(\mathbf{x}_i))}\Vert + \frac{1}{m}\sum_j\Vert{\mathbf{y}_j-\hat{G}(\hat{F}(\mathbf{y}_j))}\Vert\Bigr]\\
        \leq& \lvert\frac{1}{n} \sum_{i=1}^{n} \left(\hat{f}\circ\hat{g}\right)\left(\mathbf{x}_i\right)-\mathbb{E}_{{\hat{f}\circ\hat{g}}_{\#} \mu} [\left(\hat{f}\circ\hat{g}\right)\left(\mathbf{x}\right)]\rvert\\
        &+\lvert\frac{1}{m} \sum_{i=1}^{m} \left(\hat{g}\circ\hat{f}\right)\left(\mathbf{y}_i\right)-\mathbb{E}_{{\hat{g}\circ\hat{f}}_{\#} \nu} [\left(\hat{g}\circ\hat{f}\right)\left(\mathbf{y}\right)]\rvert\\
    \end{aligned}
    \end{equation*}
    For $\CL_{cyc}(\hat\mu,\hat\nu,\tilde{F},\tilde{G}) - \CL_{cyc}(\mu,\nu,\tilde{F},\tilde{G})$, we can write it as,
       \begin{equation*}
    \begin{aligned}
        \CL_{cyc}(\hat\mu,\hat\nu,\tilde{F},\tilde{G}) - \CL_{cyc}(\mu,\nu,\tilde{F},\tilde{G}) =
        &\Bigl[\BE_{\mathbf{x}\sim\mu}[\Vert{\mathbf{x}-\tilde{F}(\tilde{G}(\mathbf{x}))}\Vert] + \BE_{\mathbf{y}\sim\nu}[\Vert{\mathbf{y}-\tilde{G}(\tilde{F}(\mathbf{y}))}\Vert]\Bigr]\\
        &-\Bigl[\frac{1}{n}\sum_i\Vert{\mathbf{x}_i-\tilde{F}(\tilde{G}(\mathbf{x}_i))}\Vert + \frac{1}{m}\sum_j\Vert{\mathbf{y}_j-\tilde{G}(\tilde{F}(\mathbf{y}_j))}\Vert\Bigr]\\
        \leq&\lvert\frac{1}{n} \sum_{i=1}^{n} \left(\tilde{f}\circ\tilde{g}\right)\left(\mathbf{x}_i\right)-\mathbb{E}_{{\tilde{f}\circ\tilde{g}}_{\#} \mu} [\left(\tilde{f}\circ\tilde{g}\right)\left(\mathbf{x}\right)] \rvert\\
        &+\lvert\frac{1}{m} \sum_{i=1}^{m} \left(\tilde{g}\circ\tilde{f}\right)\left(\mathbf{y}_i\right)-\mathbb{E}_{{\tilde{g}\circ\tilde{f}}_{\#} \nu} [\left(\tilde{g}\circ\tilde{f}\right)\left(\mathbf{y}\right)] \rvert\\
    \end{aligned}
    \end{equation*}
    Similarly, in this case, we estimate the consistency error following a strategy similar to the proof of generalization type to process the upper bounds.
    \end{proof}

Lemma\ref{lemma:Generalization} and Lemma\ref{lemma:Cycle} come up with the upper-bounded estimation error based on the covering number of the discriminators' function classes, e.g., $\mathcal{N}\left(\epsilon, {\SD_X}_{\left.\right|_{\mathbf{\hat x}}}, \|\cdot\|_{\infty}\right)$, and the compositional function classes, e.g., $\mathcal{N}\left(\epsilon, {\CF \circ \CG}_{\left.\right|_{\mathbf{\hat x}}}, \|\cdot\|_{\infty}\right)$. As the structure of the CycleGAN is defined in Section \ref{sec:form}, the parameters of the generators' and the discriminators' neural network are bounded. We further estimate the upper bound of the covering number.
\begin{lemma}
\label{lemma:cover_num}
    Let $\mathcal{H}_1 := \mathcal{N}\mathcal{N}(\mathcal{W},\mathcal{J})$ and $\mathcal{H}_2 := \mathcal{N}\mathcal{N}(\mathcal{W^\prime},\mathcal{J})$ be the class of functions defined by a multi-layer $\operatorname{ReLU}$ neural network on $[0,1]^d$,
    $$\mathcal{H}_1=\left\{h_1: h_1(\mathbf{ x})= h_1^{[J]}, h_1^{[j]}=\sigma\left(\mathbf{A}_j^{\top} h_1^{[j-1]}+\mathbf{b}_j\right), h_1^{[0]}=\mathbf{ x}\right\}$$
    $$
    \mathcal{H}_2=\left\{h_2: h_2(\mathbf{ x})= h_2^{[J]}, h_2^{[j]}=\sigma\left(\mathbf{A}_j^{\prime\top} h_2^{[j-1]}+\mathbf{b^\prime}_j\right), h_2^{[0]}=\mathbf{ x}\right\}$$
    with the parameter constraint
$$
\begin{aligned}
\Omega_1= & \left\{ \mathbf{A}_j \in \mathbb{R}^{\mathcal{W}_{j-1} \times \mathcal{W}_j}, \mathbf{b}_j \in \mathbb{R}^{\mathcal{W}_j}:\right. \left.\max \left\{\left\|\mathbf{A}_{j,:, i}\right\|_1,\left\|\mathbf{b}_j\right\|_{\infty}\right\} \leqslant D .\right\}\\
\Omega_2= & \left\{ \mathbf{A}^\prime_j \in \mathbb{R}^{\mathcal{W^\prime}_{j-1} \times \mathcal{W^\prime}_j}, \mathbf{b}^\prime_j \in \mathbb{R}^{\mathcal{W^\prime}_j}:\right. \left.\max \left\{\left\|\mathbf{A}^\prime_{j,:, i}\right\|_1,\left\|\mathbf{b}^\prime_j\right\|_{\infty}\right\} \leqslant D .\right\}
\end{aligned}
$$
Then the covering number of $\mathcal{H}_1$ can be upper bounded as
$$
\mathcal{N}\left(\epsilon,\mathcal{H}_1, \|\cdot\|_{\infty}\right) \leq C\left(D^{J} / \epsilon\right)^M,
$$
and for the compositional function class $\mathcal{H}_1 \circ \mathcal{H}_2$,
$$
\mathcal{N}\left(\epsilon,\mathcal{H}_1 \circ \mathcal{H}_2, \|\cdot\|_{\infty}\right) \leq C\left(D^{J} / \epsilon\right)^{M^\prime},
$$
where $M=\sum_{i=1}^J \mathcal{W}_j \mathcal{W}_{j-1}+\sum_{i=1}^J \mathcal{W}_j\leq \mathcal{W}^2 \mathcal{J}$ and  $M^\prime=\sum_{i=1}^J \mathcal{W}_j \mathcal{W}_{j-1}+\sum_{i=1}^J \mathcal{W}_j+\sum_{i=1}^J \mathcal{W}^\prime_j \mathcal{W}^\prime_{j-1}+\sum_{i=1}^J \mathcal{W}^\prime_j\leq 3\mathcal{W}_{\max}^2 \mathcal{J}$ as $\mathcal{W}_{\max}=\max\{\mathcal{W},\mathcal{W}^\prime\}$ and $C$ is a constant only depending on $\mathcal{W}_j, \mathcal{J}$.

\end{lemma}
\begin{proof}
We first define the parameter constraint,
\begin{equation}
\label{eq:cn_constraint}
   \begin{aligned}
   &\Omega^\prime_1=\left\{\mathbf{A}_j \in \mathbb{R}^{d_{j-1}\times d_j} : \max_{i,j} \quad\left\|\mathbf{A}_{j, :, i}\right\|_1 \leq D\right\} \subseteq \mathbb{R}^{\sum_{j=1}^J d_{j-1} d_j}\\
&\Omega^\prime_2=\left\{\mathbf{b}_j \in \mathbb{R}^{ d_j} : \max_{j} \quad\left\|\mathbf{b}_j\right\|_{\infty} \leq D\right\} \subseteq \mathbb{R}^{\sum_{j=1}^Jd_j}\\
&\Omega^\prime_3=\left\{\mathbf{A}_j^\prime \in \mathbb{R}^{d_{j-1}\times d_j} : \max_{i,j} \quad\left\|\mathbf{A}^\prime_{j, :, i}\right\|_1 \leq D\right\} \subseteq \mathbb{R}^{\sum_{j=1}^J d_{j-1} d_j},\\
&\Omega^\prime_4=\left\{\mathbf{b}_j^\prime \in \mathbb{R}^{ d_j} : \max_{j} \quad\left\|\mathbf{b}^\prime_j\right\|_{\infty} \leq D\right\} \subseteq \mathbb{R}^{\sum_{j=1}^Jd_j}.
\end{aligned}
\end{equation}
We next analyse the bounding of $\left\|h_1^{[j]}-\tilde h_1^{[j]}\right\|_{\infty}$. For any $j=1, \cdots, \mathrm{J}$,
$$
\begin{aligned}
& \left\|h_1^{[j]}\right\|_{\infty}=\left\|\sigma\left(w_k^{\top} h_1^{[j-1]}+\mathbf{b}_k\right)\right\|_{\infty}  \leq D\left\|h_1^{[j-1]}\right\|_{\infty}+D \\
& \left\|h_1^{[j]}\right\|_{\infty}+\frac{D}{D-1} \leq D\left(\left\|h_1^{[j-1]}\right\|_{\infty}+\frac{D}{D-1}\right) \\
&\left\|h_1^{[j]}\right\|_{\infty} \leqslant D^j\left(1+\frac{D}{D-1}\right) \\
&\left\|h_1^{[J]}\right\|_{\infty} \leq 3D^J \\
\end{aligned}
$$
    $$
    \begin{aligned}
        \left\|h_1^{[j]}-\tilde h_1^{[j]}\right\|_{\infty}&=\left\|\sigma\left(\mathbf{A}_j^{\top} h_1^{[j-1]}+\mathbf{b}_j\right)-\sigma\left({ \mathbf{\tilde A}}_j^{\top} \tilde h_1^{[j-1]}+\mathbf{\tilde b}_j\right)\right\|_{\infty}\\
        & \leq\left\|\mathbf{A}_j^{\top} h_1^{[j-1]}+\mathbf{b}_j-\tilde{\mathbf{A}}^{\top}_j \tilde{h_1}^{[j-1]}-\tilde{\mathbf{b}}_j\right\|_{\infty} \\
        &\leq {\left\|\mathbf{A}_j-\tilde{\mathbf{A}}_j\right\|_1}\left\|h_1^{[j-1]}\right\|_{\infty}+\left\|{\mathbf{\tilde A}_j}\right\|_1\left\|h_1^{[j-1]}-\tilde{h_1}^{[j-1]}\right\|_{\infty}+\left\|\mathbf{b}_j-\mathbf{\tilde b}_j\right\|_{\infty}\\
        &\leq  3D^{j-1}\left\|\mathbf{A}_j-\tilde{\mathbf{A}}_j\right\|_1+D\left\|h_1^{[j-1]}-\tilde{h_1}^{[j-1]}\right\|_{\infty} +\left\|\mathbf{b}_j-\mathbf{\tilde b}_j\right\|_{\infty}\\
        &\leq JD^{(J-1)}\max_j \left(3\left\|\mathbf{A}_j-\tilde{\mathbf{A}}_j\right\|_1\right. \left.+\left\|\mathbf{b}_j-\mathbf{\tilde b}_j\right\|_{\infty}\right)
    \end{aligned}
    $$
    In this way, for the covering number $ \mathcal{N}\left(\epsilon,\mathcal{H}_1, \|\cdot\|_{\infty}\right)$ we have,
    $$
    \begin{aligned}
        &\mathcal{N}\left(\epsilon,\mathcal{H}_1, \|\cdot\|_{\infty}\right) \leq \mathcal{N}\left(\Omega^\prime_1, \frac{\epsilon}{2 JD^J},\|\cdot\|_1\right) \cdot \mathcal{N}\left(\Omega^\prime_2, \frac{\epsilon}{2 JD^J},\|\cdot\|_{\infty}\right)
    \end{aligned}
    $$
    As the parameters are defined in Eq.\eqref{eq:cn_constraint}, we can get that,
    $$
\mathcal{N}\left(\epsilon,\mathcal{H}_1, \|\cdot\|_{\infty}\right) \leq C\left(D^{J} / \epsilon\right)^M,
$$
where $M=\sum_{i=1}^J \mathcal{W}_j \mathcal{W}_{j-1}+\sum_{i=1}^J \mathcal{W}_j\leq \mathcal{W}^2 \mathcal{J}$.
For the covering number of the composition function class, it is easy to get that,
    $$
    \begin{aligned}
        &\left\|{(h_1 \circ h_2)}^{[j]}-{(\tilde h_1 \circ \tilde h_2)}^{[j]}\right\|_{\infty}\\
        =&\left\|\sigma\left(\mathbf{A}_j^{\top}(\sigma\left(\mathbf{A}_j^{\prime\top} h_2^{[j-1]}+\mathbf{b}^\prime_j\right))+\mathbf{b}_j\right)-\sigma\left({ \mathbf{\tilde A}}_j^{\top} (\sigma\left((\mathbf{\tilde A}_j^{\prime\top} \tilde h_2^{[j-1]}+ \mathbf{\tilde b}^\prime_j\right))+ \mathbf{\tilde b}_j\right)\right\|_{\infty}\\
        \leq& \left\|\left(\mathbf{A}_j^{\top}(\sigma\left((\mathbf{A}_j^{\prime\top} h_2^{[j-1]}+\mathbf{b}^\prime_j\right))+\mathbf{b}_j\right)-\left({\mathbf{\tilde A}}_j^{\top} (\sigma\left( (\mathbf{\tilde A}_j^{\prime\top} \tilde h_2^{[j-1]}+\mathbf{\tilde b}^\prime_j\right))+\mathbf{\tilde b}_j\right)\right\|_{\infty}\\
         \leq&\left\|\left(\mathbf{A}_j^{\top}(\sigma\left((\mathbf{A}_j^{\prime\top} h_2^{[j-1]}+\mathbf{b}^\prime_j\right))+\mathbf{b}_j\right)-\left({\mathbf{\tilde A}}_j^{\top} (\sigma\left((\mathbf{A}_j^{\prime\top}  h_2^{[j-1]}+ \mathbf{b}^\prime_j\right))+ \mathbf{\tilde b}_j\right)\right\|_{\infty}\\
        &+\left\|\left({\mathbf{\tilde A}}_j^{\top} (\sigma\left((\mathbf{A}_j^{\prime\top}  h_2^{[j-1]}+ \mathbf{b}^\prime_j\right))+ \mathbf{b}_j\right)-\left({\mathbf{\tilde A}}_j^{\top} (\sigma\left( (\mathbf{\tilde A}_j^{\prime\top} \tilde h_2^{[j-1]}+\mathbf{\tilde b}^\prime_j\right))+\mathbf{\tilde b}_j\right)\right\|_{\infty}\\
        \leq& {\left\|\mathbf{A}_j -\tilde{\mathbf{A}}_j\right\|_1}\left\|\sigma\left((\mathbf{A}_j^{\prime\top} h_2^{[j-1]}+\mathbf{b}^\prime_j\right)\right\|_{\infty}+2\left\|\mathbf{b}_j-\mathbf{\tilde b}_j\right\|_{\infty}\\
        &+\|\tilde{\mathbf{A}}_j \|_1\left\|\left((\mathbf{A}_j^{\prime\top}  h_2^{[j-1]}+ \mathbf{b}^\prime_j\right)-\left( (\mathbf{\tilde A}_j^{\prime\top} \tilde h_2^{[j-1]}+\mathbf{\tilde b}^\prime_j\right)\right\|_{\infty}\\
        \leq&  3D^{j-1}\left\|\mathbf{A}_j-\tilde{\mathbf{A}}_j\right\|_1+2\left\|\mathbf{b}_j-\mathbf{\tilde b}_j\right\|_{\infty}\\
        &+\left\|\mathbf{A}_j\right\|_1\left(\left\|\mathbf{A}^\prime_j-\mathbf{\tilde A}^\prime_j\right\|_1 \left\|h_2^{[j-1]}\right\|_{\infty}-\left\|\mathbf{\tilde A}^\prime_j\right\|_1\left\|{h_2}^{[j-1]}-\tilde h_2^{[j-1]}\right\|_{\infty}+\left\|\mathbf{b}^\prime_j-\mathbf{\tilde b}^\prime_j\right\|_{\infty}\right)\\
        \leq &4JD^{(J+1)}\max_j \left(\left\|\mathbf{A}_j-\tilde{\mathbf{A}}_j\right\|_1+\left\|\mathbf{A}^\prime_j-\tilde{\mathbf{A}}_j^\prime\right\|_1+\left\|\mathbf{b}_j-\mathbf{\tilde b}_j\right\|_{\infty}+\left\|\mathbf{b}^\prime_j-\mathbf{\tilde b}^\prime_j\right\|_{\infty}\right)
    \end{aligned}
    $$
    As we have,
     $$
    \begin{aligned}
        &\mathcal{N}\left(\epsilon,\mathcal{H}_1 \circ \mathcal{H}_2, \|\cdot\|_{\infty}\right) \\
        &\leq \mathcal{N}\left(\Omega^\prime_1, \frac{\epsilon}{16 JD^J},\|\cdot\|_1\right) \cdot \mathcal{N}\left(\Omega^\prime_2, \frac{\epsilon}{16 JD^J},\|\cdot\|_{\infty}\right)\\
        & \quad \times \mathcal{N}\left(\Omega^\prime_3, \frac{\epsilon}{16 JD^J},\|\cdot\|_1\right)\cdot \mathcal{N}\left(\Omega^\prime_4, \frac{\epsilon}{16 JD^J},\|\cdot\|_{\infty}\right)
    \end{aligned}
    $$
    Following the parameters defined in Eq.\eqref{eq:cn_constraint}, we can get that,
    $$
\mathcal{N}\left(\epsilon,\mathcal{H}_1 \circ \mathcal{H}_2, \|\cdot\|_{\infty}\right) \leq C\left(D / \epsilon\right)^{M^\prime},
$$
where $M^\prime=\sum_{i=1}^J \mathcal{W}_j \mathcal{W}_{j-1}+\sum_{i=1}^J \mathcal{W}_j+\sum_{i=1}^J \mathcal{W}^\prime_j \mathcal{W}^\prime_{j-1}+\sum_{i=1}^J \mathcal{W}^\prime_j\leq 3\mathcal{W}_{\max}^2 \mathcal{J}$ as $\mathcal{W}_{\max}=\max\{\mathcal{W},\mathcal{W}^\prime\}$.
\end{proof}

We then provide the upper bound of the estimation error based on the bounded parameter sets of the generator and discriminator networks as Lemma \ref{lemma:est2}.
\begin{lemma}
\label{lemma:est2}
  Let $\mu, \nu$ be the target distribution over the compact domain $X$, $Y$ on $[0,1]^d$, given $n$ i.i.d training samples as $\left\{\mathbf{x}_i\right\}_{i=1}^{n}$ from $\mu$ and $m$ i.i.d training samples $\left\{\mathbf{y}_i\right\}_{i=1}^{m}$ from $\nu$. Let $\SD_X = \mathcal{N}\mathcal{N}(\mathcal{W}_{D_X},\mathcal{L},1)$ and $\SD_{Y} = \mathcal{N}\mathcal{N}(\mathcal{W}_{D_Y},\mathcal{L},1)$ be the function classes of discriminators $D_X,D_Y$, $\CF = \mathcal{N}\mathcal{N}(\mathcal{W}_F,\mathcal{L},B_{F})$ and $\CG = \mathcal{N}\mathcal{N}(\mathcal{W}_G,\mathcal{L},B_{G})$ be the function classes of generators $F,G$ as defined in Section \ref{sec:form}. Then, with probability $(1-12\delta)$ over randomness of the training samples and $\lambda>0$,
  \begin{equation*}
    \begin{aligned}
        L(\hat{F},\hat{G})-L(\tilde{F},\tilde{G})
		\leq  C B (\sqrt{\frac{\mathcal{W}^2\mathcal{L}}{m}}+\sqrt{\frac{\mathcal{W}^2\mathcal{L}}{n}}+\sqrt{\frac{\log \frac{1}{\delta}}{ m}}+ \sqrt{\frac{\log \frac{1}{\delta}}{ n}})
    \end{aligned}
\end{equation*}
where $\mathcal{W} := \max\{\mathcal{W}_{D_X},\mathcal{W}_{D_Y},\mathcal{W}_F,\mathcal{W}_G\}$ and $B:= \max\{ B_F, B_G\}$.
\end{lemma}
  \begin{proof}
We take the covering number of discriminator function classes $\mathcal{N}\left(\epsilon, {\SD_X}_{\left.\right|_{\mathbf{\hat x}}}, \|\cdot\|_{\infty}\right)$ into consideration first.
Following the analysis of Lemma \ref{lemma:cover_num}, we can get that the covering number of $\SD_X$ can be bounded as,
$$
\mathcal{N}\left(\epsilon,\SD_X, \|\cdot\|_{\infty}\right) \leq C\left(1 / \epsilon\right)^{\mathcal{W}_{D_X}^2 \mathcal{L}}
$$
Similarly, we can show that the compositional function classes $\mathcal{N}\left(\epsilon, {{\SD_{X}} \circ \CF}_{\left.\right|_{\mathbf{\hat y}}}, \|\cdot\|_{\infty}\right)$ can also be bounded as $\mathcal{W}_{\max}=\max\{\mathcal{W}_{D_X},\mathcal{W}_{F}\}$,
$$
\mathcal{N}\left(\epsilon,\SD_X, \|\cdot\|_{\infty}\right) \leq C\left(B_F / \epsilon\right)^{3(\mathcal{W}_{\max})^2 (\mathcal{L})}
$$
Apply the upper bound of the covering number to further bounding Lemma \ref{lemma:Generalization},
    \begin{align*}
        \CL_{\SD_X, \CF}(\mu,\nu) &\leq 16 \mathbb{E}_{\mathbf{\hat y}} \inf _{0<{\xi_{1}}< {B_F} / 2}\left( {\xi_{1}}+\frac{3}{\sqrt{m}} \int_{\xi_1}^{ {B_F} / 2} \sqrt{\log \mathcal{N}\left(\epsilon, {{\SD_{X}} \circ \CF}_{\left.\right|_{\mathbf{\hat y}}}, \|\cdot\|_{\infty}\right)} d \epsilon\right) \\
        &\quad + 32 \mathbb{E}_{\mathbf{\hat x}} \inf _{0<{\xi_{2}}< 1 / 2}\left({\xi_{2}}+\frac{3}{\sqrt{n}} \int_{\xi_2}^{1 / 2} \sqrt{\log \mathcal{N}\left(\epsilon, {\SD_X}_{\left.\right|_{\mathbf{\hat x}}}, \|\cdot\|_{\infty}\right)} d \epsilon\right)\\
        &\quad +2 B_F \sqrt{\frac{2\log \frac{1}{\delta_1}}{ m}}+ 2 \sqrt{\frac{2\log \frac{1}{\delta_2}}{ n}}\\
        &\leq  \inf _{0<{\xi_{1}}< {B_F} / 2}\left(16 \xi_1+48 \sqrt{\frac{C_1\mathcal{W}_{\max}^2 \mathcal{L}}{m}} \int_{\xi_1}^{B_F / 2} \sqrt{\log ( B_F / \epsilon)} d \epsilon\right)\\
        &\quad +\inf _{0<{\xi_{2}}< {1} / 2}\left(32 \xi_2+96 \sqrt{\frac{C_2\mathcal{W}_{D_X}^2 \mathcal{L}}{n}} \int_{\xi_2}^{1 / 2} \sqrt{\log (1 / \epsilon)} d \epsilon\right)\\
         &\quad +2 B_F \sqrt{\frac{2\log \frac{1}{\delta_1}}{ m}}+ 2 \sqrt{\frac{2\log \frac{1}{\delta_2}}{ n}}\\
        &\leq  \inf _{0<{\xi_{1}}< {B_F} / 2}\left(16 \xi_1+24 B_F\sqrt{\frac{C_1\mathcal{W}_{\max}^2 \mathcal{L}\log (B_F  / \xi_1)}{m}}\right)\\
        &\quad +\inf _{0<{\xi_{2}}< {1} / 2}\left(32 \xi_2+48 \sqrt{\frac{C_2\mathcal{W}_{D_X}^2 \mathcal{L}\log (1 / \xi_{2})}{n}} \right)\\
         &\quad +2 B_F \sqrt{\frac{2\log \frac{1}{\delta_1}}{ m}}+ 2  \sqrt{\frac{2\log \frac{1}{\delta_2}}{ n}}\\
& \leq C \left\{B_F (\sqrt{\frac{\mathcal{W}_{\max}^2 \mathcal{L}}{m}}+\sqrt{\frac{\log \frac{1}{\delta_1}}{ m}})+\sqrt{\frac{\mathcal{W}_{D_X}^2 \mathcal{L}}{n}}+ \sqrt{\frac{\log \frac{1}{\delta_2}}{ n}}\right\},
    \end{align*}.

The result of estimation error in the cycle-consistency type in Lemma \ref{lemma:Cycle} can be analyzed following a similar strategy. We then can give the upper bound of the forward generation process similarly in view of the symmetry in structure. Combining the upper bounds of the backward and forward translation, we summarize the main result of the estimation error bounding and get the result with a probability of $1-12 \delta$,
\begin{equation*}
    \begin{aligned}
        L(\hat{F},\hat{G})-L(\tilde{F},\tilde{G})
		\leq  C' &\{B \{(2\lambda B\mathcal+2)(\sqrt{\frac{\mathcal{W}^2\mathcal{L}}{m}}+\sqrt{\frac{\mathcal{W}^2\mathcal{L}}{n}})+\sqrt{\frac{\log \frac{1}{\delta}}{ m}}+ \sqrt{\frac{\log \frac{1}{\delta}}{ n}}\}\\
        +&\{(2\lambda B\mathcal+2)(\sqrt{\frac{\mathcal{W}^2\mathcal{L}}{m}}+\sqrt{\frac{\mathcal{W}^2\mathcal{L}}{n}})+\sqrt{\frac{\log \frac{1}{\delta}}{ m}}+ \sqrt{\frac{\log \frac{1}{\delta}}{ n}})\}\} \\
        \leq  C &B \{(2\lambda B\mathcal+2)(\sqrt{\frac{\mathcal{W}^2\mathcal{L}}{m}}+\sqrt{\frac{\mathcal{W}^2\mathcal{L}}{n}})+\sqrt{\frac{\log \frac{1}{\delta}}{ m}}+ \sqrt{\frac{\log \frac{1}{\delta}}{ n}}\}
    \end{aligned}
\end{equation*}
where $\mathcal{W} := \max\{\mathcal{W}_{D_X},\mathcal{W}_{D_Y},\mathcal{W}_F,\mathcal{W}_G\}$ and $B:= \max\{B_F, B_G\}$.
  \end{proof}
  Based on Lemma \ref{lemma:est2} and set $\lambda = \frac{1}{B}$, we can get that,
\begin{equation*}
    \begin{aligned}
        L(\hat{F},\hat{G})-L(\tilde{F},\tilde{G})
          &=O(B (\sqrt{\frac{\mathcal{W}^2\mathcal{L}}{m}}+\sqrt{\frac{\mathcal{W}^2\mathcal{L}}{n}}+\sqrt{\frac{\log \frac{1}{\delta}}{ m}}+ \sqrt{\frac{\log \frac{1}{\delta}}{ n}})).
    \end{aligned}
\end{equation*}
So, we complete the proof of Theorem \ref{thm:est2}.

\end{appendices}

\end{document}